\def\isarxiv{1} 

\ifdefined\isarxiv
\documentclass[11pt]{article}

\usepackage[numbers]{natbib}

\else

\documentclass[twoside]{article}

\usepackage[accepted]{aistats2025}
\usepackage[round]{natbib}

\fi

\usepackage{amsmath}
\usepackage{amsthm}
\usepackage{amssymb}
\usepackage{algorithm}
\usepackage{subfig}
\usepackage{algpseudocode}
\usepackage{graphicx}
\usepackage{grffile}
\usepackage{wrapfig,epsfig}
\usepackage{url}
\usepackage{xcolor}
\usepackage{epstopdf}
\usepackage{colortbl} 

\usepackage{bbm}
\usepackage{dsfont}

\allowdisplaybreaks

\DeclareSymbolFont{extraup}{U}{zavm}{m}{n}
\DeclareMathSymbol{\varheart}{\mathalpha}{extraup}{86}
\DeclareMathSymbol{\vardiamond}{\mathalpha}{extraup}{87}

\ifdefined\isarxiv

\renewcommand*{\citet}{\cite} 
\renewcommand*{\citep}{\cite}

\usepackage{tikz}
\usepackage{hyperref}  
\hypersetup{colorlinks=true,citecolor=blue,linkcolor=blue} 
\usetikzlibrary{arrows}
\usepackage[margin=1in]{geometry}

\else
\usepackage{tikz}
\usepackage{hyperref}
\usetikzlibrary{arrows}
\fi
 
\graphicspath{{./figs/}}

\theoremstyle{plain}
\newtheorem{theorem}{Theorem}[section]
\newtheorem{lemma}[theorem]{Lemma}
\newtheorem{definition}[theorem]{Definition}

\newtheorem{remark}[theorem]{Remark}

\newcommand{\R}{\mathbb{R}}

\newcommand{\0}{\mathbf{0}}
\newcommand{\1}{\mathbf{1}} 
\newcommand{\flag}{\mathrm{flag}}
\newcommand{\EOF}{\mathrm{EOF}}
\newcommand{\mem}{\texttt{mem}}
\newcommand{\target}{\mathrm{target}}

\definecolor{lightblue}{HTML}{BCF2F6}
\definecolor{lightgreen}{HTML}{C0EBA6}
\definecolor{lightred}{HTML}{FFAAAA}
\definecolor{lightpurple}{HTML}{E4B1F0}
\definecolor{lightyellow}{HTML}{FFF100}
\definecolor{olive}{rgb}{0.69, 0.61, 0.85}
\definecolor{darkblue}{HTML}{024CAA}
\definecolor{darkgreen}{HTML}{185519}
\definecolor{darkred}{HTML}{C62E2E}

\newcommand{\bluecell}{\cellcolor{lightblue}}

\newcommand{\greencell}{\cellcolor{lightgreen}}
\newcommand{\redcell}{\cellcolor{lightred}}

\DeclareMathOperator{\poly}{poly}

\newcommand{\ReLU}{\mathsf{ReLU}} 
\newcommand{\SUBLEQ}{\texttt{SUBLEQ}}
\newcommand{\ReLUMLP}{\textsf{ReLU}-\textsf{MLP}~} 
\newcommand\tikznode[3][]%
   {\tikz[remember picture,baseline=(#2.base)]
      \node[minimum size=0pt,inner sep=0pt,#1](#2){#3};%
   }

\makeatletter
\newcommand*{\RN}[1]{\expandafter\@slowromancap\romannumeral #1@}
\makeatother

\usepackage{lineno}

\begin{document}

\ifdefined\isarxiv

\date{}








\title{
Looped ReLU MLPs May Be All You Need as Practical Programmable Computers
}
\author{
Yingyu Liang\thanks{\texttt{
yingyul@hku.hk}. The University of Hong Kong. \texttt{
yliang@cs.wisc.edu}. University of Wisconsin-Madison.} 
\and
Zhizhou Sha\thanks{\texttt{ shazz20@mails.tsinghua.edu.cn}. Tsinghua University.}
\and
Zhenmei Shi\thanks{\texttt{
zhmeishi@cs.wisc.edu}. University of Wisconsin-Madison.}
\and 
Zhao Song\thanks{\texttt{ magic.linuxkde@gmail.com}. The Simons Institute for the Theory of Computing at the UC, Berkeley.}
\and 
Yufa Zhou\thanks{\texttt{ yufazhou@seas.upenn.edu}. University of Pennsylvania.}
}

\else

\runningtitle{Looped ReLU MLPs May Be All You Need as Practical Programmable Computers}

\twocolumn[

\aistatstitle{
Looped ReLU MLPs May Be All You Need as Practical Programmable Computers
}


\aistatsauthor{ 
Yingyu Liang$^{\vardiamond,\varheart}$
\And 
Zhizhou Sha$^{\clubsuit}$
\And  
Zhenmei Shi$^{\vardiamond}$
\And 
Zhao Song$^{\spadesuit}$
\And 
Yufa Zhou$^\heartsuit$
}


\aistatsaddress{ 
$^\vardiamond$University of Wisconsin-Madison. 
\qquad
$^\varheart$The University of Hong Kong. 
\\
$^\clubsuit$Tsinghua University. 
\qquad
$^\heartsuit$University of Pennsylvania. 
\\
$^\spadesuit$The Simons Institute for the Theory of Computing at the University of California, Berkeley. 
} ]


\fi

\ifdefined\isarxiv
\begin{titlepage}
  \maketitle
  \begin{abstract}
Previous work has demonstrated that attention mechanisms are Turing complete. More recently, it has been shown that a looped 9-layer Transformer can function as a universal programmable computer. In contrast, the multi-layer perceptrons with $\mathsf{ReLU}$ activation ($\mathsf{ReLU}$-$\mathsf{MLP}$), one of the most fundamental components of neural networks, is known to be expressive; specifically, a two-layer neural network is a universal approximator given an exponentially large number of hidden neurons. However, it remains unclear whether a $\mathsf{ReLU}$-$\mathsf{MLP}$ can be made into a universal programmable computer using a practical number of weights. In this work, we provide an affirmative answer that a looped 23-layer $\mathsf{ReLU}$-$\mathsf{MLP}$ is capable of performing the basic necessary operations, more efficiently and effectively functioning as a programmable computer than a looped Transformer. This indicates simple modules have stronger expressive power than previously expected and have not been fully explored. Our work provides insights into the mechanisms of neural networks and demonstrates that complex tasks, such as functioning as a programmable computer, do not necessarily require advanced architectures like Transformers.

  \end{abstract}
  \thispagestyle{empty}
\end{titlepage}

{\hypersetup{linkcolor=black}
\tableofcontents
}
\newpage

\else

\begin{abstract}
Previous work has demonstrated that attention mechanisms are Turing complete. More recently, it has been shown that a looped 9-layer Transformer can function as a universal programmable computer. In contrast, the multi-layer perceptrons with $\mathsf{ReLU}$ activation ($\mathsf{ReLU}$-$\mathsf{MLP}$), one of the most fundamental components of neural networks, is known to be expressive; specifically, a two-layer neural network is a universal approximator given an exponentially large number of hidden neurons. However, it remains unclear whether a $\mathsf{ReLU}$-$\mathsf{MLP}$ can be made into a universal programmable computer using a practical number of weights. In this work, we provide an affirmative answer that a looped 23-layer $\mathsf{ReLU}$-$\mathsf{MLP}$ is capable of performing the basic necessary operations, more efficiently and effectively functioning as a programmable computer than a looped Transformer. This indicates simple modules have stronger expressive power than previously expected and have not been fully explored. Our work provides insights into the mechanisms of neural networks and demonstrates that complex tasks, such as functioning as a programmable computer, do not necessarily require advanced architectures like Transformers.
\end{abstract}

\fi

\section{INTRODUCTION}

Transformers \citep{vsp+17} have demonstrated their potential across a variety of tasks, emerging as a dominant choice for a wide spectrum of practical applications, including natural language processing (NLP) \citep{dclt19, rsr+20, gpt4turbo, phi3, claude3.5, o1, llama3} and computer vision \citep{dbk+20, px23, hwc+22, adh+21, knh+22, wsd+23, wcz+23, wxz+24,lssz24_gm, wms+24}, among others. The success of Transformers is largely attributed to their ability to perform complex operations, such as induction head \citep{oen+22, cs24}, in-context learning \citep{dld+22, wbz+21, wtb+22, xsl24, swxl24, cll+25_icl}, information retrieval~\cite{smn+24} and chain of thoughts \citep{wws+22,kgr+22}. Meanwhile, another line of research explores the theoretical capability of Transformers. For instance, \citet{pbm21} has proven the Turing completeness of the attention mechanism.
However, Turing completeness is not a feature unique to Transformers. \citet{ss92,cs21} has demonstrated that Recurrent Neural Networks (\textsf{RNN}s) are also Turing complete.   
Other works \citep{p99, kl20} have shown that the most basic module in deep learning, the Multi-Layer Perceptron (\textsf{MLP}), is a universal approximator. 

However, the concept of Turing completeness inherently necessitates infinite memory, an impracticality in real-world scenarios due to the finite nature of available memory. 
In \cite{grs+23}, they bridge the gap between the theoretical Turing machine and the practical Transformer-based programmable computer by illustrating that a looped 9-layer Transformer possesses the necessary expressiveness to operate as a programmable computer. 
Given that \ReLUMLP also has the capability for achieving Turing completeness and itself is a universal approximator, this raises an interesting question:
\begin{center}
    \textit{Is \ReLUMLP expressive enough to be a practical programmable computer?}
\end{center}

To the best of our knowledge, no practical solution for constructing a \ReLUMLP as a general-purpose computer has been proposed before. Therefore, we explore the conditions under which a \ReLUMLP can function as a universal programmable computer and prove that a $23$-layer looped \ReLUMLP is capable of emulating a general-purpose computer.
Our main approach involves constructing a $23$-layer \ReLUMLP to demonstrate that the minimalistic SUBLEQ instruction can be implemented, thereby showing that the computational power of a $23$-layer \ReLUMLP is comparable to that of a programmable computer. This idea was first introduced by \cite{grs+23}, but our key contribution lies in using a simpler \ReLUMLP architecture, in contrast to the more complex Transformer architecture employed in \cite{grs+23}. 

Our research centers on $\mathsf{MLP}$s equipped with the $\mathsf{ReLU}$ activation function, which is fundamental to their design. We commence by formally defining the $\mathsf{ReLU}$ activation function as follows:
\begin{definition}[$\ReLU$] \label{def:relu}
We define \textsf{ReLU} as follows:
For a vector $x \in \R^n$, the output of the $i$-th entry of \textsf{ReLU} for $i \in \{1,2, \dots, n\}$ is 
$\ReLU(x)_i = \max \{x_i,0\}.$
\end{definition}
Then, we introduce the definition of $\mathsf{MLP}$ with $\ReLU$ activation as follows:
\begin{definition}[\textsf{ReLU}-\textsf{MLP}] \label{def:relu_mlp}
Let $n$ be the size of the state.
Let input be $x \in \R^n$. The weights and biases are $W \in \R^{n \times n}, b \in \R^n$.
We have the $1$-layer output of \textsf{ReLU} Multiple-Layer Perceptron (\textsf{ReLU-MLP}) as
\begin{align} \label{eq:relu_mlp}
    \ReLU(W x + b) \in \R^n.
\end{align}
The $m$-layer \ReLUMLP is then the compositions of $m$ such Perceptrons in Eq.~\eqref{eq:relu_mlp}, e.g., $\ReLU(W_2 \cdot  \ReLU(W_1 x + b_1) + b_2)$ for $2$-layer \textsf{ReLU-MLP}.
\end{definition}
Based on this setting, we have demonstrated that a $23$-layer looped \ReLUMLP is capable of emulating a general-purpose computer.
This is accomplished by constructing a $23$-layer \ReLUMLP that can execute a generalized form of the single instruction \SUBLEQ~(see also Algorithm~\ref{alg:subleq}). The \SUBLEQ~instruction, which stands for Subtract and Branch if Less-than or Equal to zero, operates on three addresses $a, b, c \in \R^{\log n}$. It subtracts the value at memory location $\mem[b]$ from the value at memory location $\mem[a]$, stores the result back at $\mem[b]$, and if the result is less than or equal to zero, the program execution continues at the address specified by $c$. Otherwise, it will execute the next instruction without branching. 
Despite the instruction's simplicity, it is powerful enough to be the foundation of a universal computing system \citep{mp88, subleq}. Consequently, we have established that our \ReLUMLP design constitutes a functional One Instruction Set Computer (OISC).

In contrast to previous research \citep{grs+23}, which utilized Transformers as the fundamental building block to create a universal computer, our approach harnesses the simple \ReLUMLP to accomplish the same objective.
Furthermore, for each forward pass, our looped \textsf{ReLU-MLP} takes $O(n \log n)$ time complexity, while looped Transformer takes $O(n^2)$ (Section~\ref{sec:discussion:comparision_with_attention}).
These findings suggest that basic \ReLUMLP modules are sufficiently expressive, and their potential is underexplored. With careful design, they might exhibit emergent abilities like in-context learning, indicating that complex architectures like Transformers may not always be necessary for certain computational tasks. Understanding the fundamental capabilities of neural networks before adopting complex models could lead to more efficient, less resource-intensive solutions. Challenging the belief that only advanced models can handle complex tasks opens avenues for research, prioritizing simplicity and efficiency without sacrificing performance. In a world of finite computational resources, this discovery could lead to more sustainable and accessible AI solutions.

To sum up, we conclude our contributions as follows:
\begin{itemize}
    \item To the best of our knowledge, we are the first work to prove that a looped $23$-layer \textsf{ReLU-MLP} satisfies the conditions required to function as programmable computers (Theorem~\ref{thm:looped_relu_mlp_as_programmable_computer:informal}). 
    \item For each forward pass, our looped \textsf{ReLU-MLP} takes $O(n \log n)$ time complexity, while looped Transformer takes $O(n^2)$ (Section~\ref{sec:discussion:comparision_with_attention}), highlighting the importance of understanding the capabilities of the fundamental \ReLUMLP component and demonstrating its powerful expressivity.
    \item Our findings show that traditional neural networks, such as \textsf{ReLU-MLP}, have not been fully explored, challenging the belief that only advanced architectures can perform complex tasks. 
\end{itemize}

\paragraph{Roadmap.}
The paper is organized as follows: In Section~\ref{sec:related}, we discuss related literature.
In Section~\ref{sec:preli}, we provide our notation system and key concepts and definitions. 
In Section~\ref{sec:main}, we introduce our main result that a looped $23$-layer \ReLUMLP can emulate a programmable computer.
In Section~\ref{sec:key_func}, we demonstrate how to implement basic operations such as read, write, conditional branching, and \SUBLEQ~using \textsf{ReLU-MLP}.
In Section~\ref{sec:discuss}, we discuss the high-level intuition and potential future directions of our findings.
In Section~\ref{sec:conclusion}, we conclude our paper. 
\section{RELATED WORK}\label{sec:related}

\paragraph{Complexity and Neural Networks.}

Circuit complexity, a branch of computational complexity theory, studies circuit families as models of computation\footnote{We refer the reader to the chapter $6$ and $14$ of \citet{ab09} or chapter $1$ and $2$ of \textit{Handbook of Theoretical Computer Science} \citep{b14, j90} for more detailed background of circuit complexity.}. Several circuit complexity classes are significant in machine learning. Specifically, $\mathsf{AC}^0$ represents problems highly parallelizable with standard logic gates, while $\mathsf{TC}^0$ extends this to include \textit{threshold gates}, and $\mathsf{NC}^1$ denotes the language recognizable by $O(\log n)$-depth circuits with bounded gate arity \citep{mss22}. It is known that $\mathsf{AC}^0 \subset \mathsf{TC}^0 \subseteq \mathsf{NC}^1$, but whether $\mathsf{TC}^0 \neq \mathsf{NC}^1$ remains an open question. Assuming this inequality, \citet{lag+22} shows that Transformer depth must depend on input sequence length when simulating non-solvable semiautomata.
\citet{llzm24} explore relationships among constant-depth Transformers, Transformers with Chain-of-Thought (CoT), and circuit complexity. They demonstrate: 
$\mathsf{T}[\poly (n), 1, 1] \subseteq  \mathsf{CoT}[\log n, \poly(n), 1, 1] \subseteq \mathsf{AC}^0$ and 
$\mathsf{T}[\poly(n), \log n, 0] \subseteq  \mathsf{CoT}[\log n, \poly(n), \log n, 0] 
    \subseteq  \mathsf{TC}^0$
where $\mathsf{T}[d(n), s(n), e(n)]$ denotes a constant-depth Transformers with embedding size $d(n)$, precision $s(n)$ bits, and exponent bits $e(n)$ for input length $n$ and $\mathsf{CoT}[T(n), d(n), s(n), e(n)]$ denotes a $T(n)$-step CoT of a constant-depth Transformer $\mathsf{T}[d(n), s(n), e(n)]$. It provides theoretical insights into the emergent CoT ability of Transformers, showing that intermediate reasoning steps enable tackling more complex problems.

The Strong Exponential Time Hypothesis ({\sf SETH}), introduced by \citet{ip01}, strengthens the $\mathsf{P} \neq \mathsf{NP}$ conjecture by asserting that current best $\mathsf{SAT}$ algorithms are roughly optimal: for every $\epsilon > 0$, there exists $k \geq 3$ such that $k$-$\mathsf{SAT}$ cannot be solved in $O(2^{(1-\epsilon)n})$ time, even randomly. {\sf SETH} is widely used to prove fine-grained lower bounds for various algorithmic problems~\citep{w18} and has been applied to derive lower bounds for Transformer training/inference~\citep{as23, as24_neurips,lss+24} and tensor attention~\citep{as24_iclr, lssz24_tat}. 
Specifically, \citet{as23} demonstrates that unless the $\mathsf{SETH}$ fails, no algorithm exists that can compute the forward pass of an attention network in truly subquadratic time. On the other hand, \citet{as24_neurips} establishes that the same condition applies to the backward computation of attention networks, i.e., unless the $\mathsf{SETH}$ fails, no truly-subquadratic time algorithm can be devised for the backward computation of attention networks.
In essence, complexity theory provides a powerful framework for investigating neural networks' computational capabilities by rigorously analyzing the computational problems they can efficiently solve.

\paragraph{Turing Completeness of Neural Networks.}

In recent years, neural networks ($\mathsf{NN}$s) have demonstrated great potential in performing tasks that were previously considered impossible for traditional numerical approximation methods. This remarkable capability is largely attributed to their properties as universal approximators \citep{p99,ybr+19,kl20,cll+25_var,hwl+24} and, in some cases, their Turing completeness \citep{ss92,pmb19,dgv+19,cs21,pbm21,smg24}. Specifically, \citet{pmb19,pbm21} show that Transformers with attention mechanism under infinite precision are Turing complete, whereas \citet{dgv+19} demonstrates that this is not the case under fixed precision. 
Another line of work \citep{p87,ss92,i95,ks96,cs21,smg24} focuses on recurrent neural networks ($\mathsf{RNN}$s) and proves their Turing completeness. 
Moreover, \citet{wcm22} demonstrates that \ReLUMLP can meaningfully approximate Boolean circuits, and Transformers can meaningfully approximate Turing machines. It is important to note that Turing completeness deals with discrete computations, such as processing language, whereas universal approximation focuses on continuous functions. 
\citet{swl21,swl24} show that \ReLUMLP is expressive and over fixed feature methods like kernels.
Thus, one property does not imply the other \citep{ybr+19}, and it is necessary to study these two subjects separately.

\paragraph{Limitations of Transformers.}
Transformers have demonstrated remarkable capability in natural language processing tasks, yet their proficiency in mathematical computations remains a concern \citep{c22}. Therefore, research has been directed toward delineating the computational limits of Transformers when faced with mathematical tasks. \citet{ms23} has shown that if $\mathsf{L} \neq \mathsf{P}$ \footnote{The class $\mathsf{L}$ represents the set of problems that can be resolved using logarithmic space, whereas $\mathsf{P}$ denotes the class of problems that can be solved within polynomial time constraints.} (i.e. not all polynomial-time problems are solvable in logarithmic space), Transformers are incapable of accurately resolving linear inequalities or determining membership in an arbitrary context-free grammar that includes empty productions, and \citet{fzg+24} illustrates that unless $\mathsf{TC}^0 = \mathsf{NC}^1$, there is no log-precision Transformers is capable to solve arithmetic and equation-solving problems. \cite{lls+25_graph,kll+25_var,kll+25_tc,kls+25,cll+25_mamba,lls+24_tensor,lll+24,cll+24_rope,hlsl24,hsk+24,hwg+24} show the limitation of Transformers by circuit complexity or some other frameworks.

\paragraph{Neural Networks Can Perform Algorithms.}
Given their Turing completeness, it is not surprising that neural networks ($\mathsf{NN}$s) can perform algorithms once properly trained. One example is their ability for in-context learning (ICL) \citep{oen+22, mlh+22, xsl24, swxl24, gsx23}, where Transformers produce the correct output based on the context provided by examples without adapting their parameters. Studies have shown that ICL can implement optimization algorithms like gradient descent across layers \citep{vkr+23, asa+23,acds24, mhm23, gsr+24}, and interestingly, Transformers can in-context fine-tune smaller Transformers \citep{pmxa24}.
Moreover, \citet{llzm24} shows that when equipped with enough steps of Chain-of-Thought (CoT) reasoning \citep{wws+22, kgr+22}, constant-depth Transformers using constant-bit precision and embedding size can solve any problem solvable by Boolean circuits. 
Other studies \citep{zpga23, al23, lll+25_loop, lls+25_grok} observe that Transformers perform dynamic programming to generate.
Transformers have also been shown to efficiently learn arithmetic operations such as addition, multiplication, and elementary functions like square roots, and can even simulate a programmable computer \citep{lsl+23, grs+23}. 
Additionally, \citet{hs24} proves that \ReLUMLP can solve exact max-flow problems.

\paragraph{Neural Networks as Practical Programmable Computer.}

The programmable computer is known as a powerful and controllable computing architecture. Numerous studies strive to establish the equivalence of their proposed neural network architectures with the programmable computer, thereby illustrating the efficacy of their designs. \citet{grs+23} employs Transformers as the building block to build a programmable computer, showcasing the latent capabilities of Transformer-based neural networks. 
Additionally, other research initiatives adopt distinct methodologies to realize programmable computer architectures. For example, \citet{cdt24} introduces a construction utilizing optical neural networks. Studies such as \citet{lv19, l21} are dedicated to investigating the feasibility of attaining universal computation through probabilistic circuits.
Moreover, \cite{wgy21} proposes a computational model for the Transformer-encoder using a domain-specific language called the Restricted Access Sequence Processing Language (RASP). 
Building on this, \cite{lkf+24} introduce Tracr, a compiler that leverages RASP to use Transformer networks as programmable units.

\section{PRELIMINARY}\label{sec:preli}

This section provides essential definitions used in this paper.
In Section~\ref{sec:preliminary:notations}, we introduce some basic notations.
In Section~\ref{sec:preliminary:key_concepts}, we present several key concepts related to the state vector of the programmable computer constructed by \textsf{ReLU}-\textsf{MLP}.

\subsection{Notations} \label{sec:preliminary:notations}

We use $[n]$ to denote $\{1,2,\dots,n\}$ for any $n \in \mathbb{N}_+$.
We use $e_i$ to denote a vector in which only the $i$-th location is $1$ and zeros everywhere else.
We denote an all $1$ vector using ${\bf 1}_{n} \in \R^n$.
We denote an all $0$ vector using ${\bf 0}_{n} \in \R^n$.
We use $a^\top b$ to denote the inner product of $a,b \in \R^d$ i.e. $a^\top b := \sum_{i=1}^d a_i b_i$.
We use $\circ$ to denote the Hadamard product, i.e., the $i$-th entry of $a \circ b$ is $a_{i} b_{i}$.
Let $I_{d \times d} \in \R^{d \times d}$ denote an identity matrix.

\subsection{Key Concepts} \label{sec:preliminary:key_concepts}

We begin by introducing the way we organize the data. Different from conventional $\{0, 1\}$ representation of the data, we use $\{\pm 1\}$ to represent the data. This design will benefit the calculation of the address vectors, which will be discussed in Remark~\ref{rem:address_vector_property}. 

\begin{definition}[One-Bit Data]
We define one-bit data as $v \in \{-1, 1\}$.
\end{definition}

One bit is not capable of representing the integers or floats or other data types used in modern computers. Thus, we introduce the $d$-bits data vector as follows:
\begin{definition}[$d$-Bits Data]
We define data as $v \in \{ -1, 1\}^d$ using two's complement. Here, data dimension $d$ means the number of bits, and the data type can be int32 or float64 in a computer.
\end{definition}

In this work, we consider all data as integers. Specifically, we use $2$'s complement to represent the integer. It is worth mentioning that the data type can be easily extended. Due to the space limitation, we temporarily consider only integer data.

\begin{definition}[$2$’s Complement]\label{def:complement}
For a $d$-bit data value $v \in \{ -1, 1\}^d$, which represents an integer with bits $b_d, b_{d-1}, \ldots, b_2, b_1$, where $b_i \in \{\pm 1\}$ for $i \in [d]$, we denote $b_d$ as the most significant bit (MSB).

The integer value of $v$ is defined as follows:
\begin{itemize}
    \item If $b_d = -1$, the integer is considered positive, with a value given by: $\sum_{i=1}^{d-1} 2^{i-1} \frac{b_i+1}{2}$.
    \item If $b_d = +1$, the integer is considered negative, with a value given by: $- 2^{d-1} + \sum_{i=1}^{d-1} 2^{i-1} \frac{b_i+1}{2}$.
\end{itemize}

\end{definition}

Suppose there are total $n$ bits in the programmable computer. Hence, we choose the length of the address vector as $\log(n)$, which is the most efficient way to locate the total $n$ address. We present the definition of address vector as follows:
\begin{definition}[Address] \label{def:address_vector}
We define the address of a data as $a \in \{-1, +1\}^{\log (n)}$ with state size $n$.
\end{definition}

Since we choose $\{\pm 1\}$ instead of $\{0, 1\}$ as our data representation, only the inner product of the address vectors with the same address will be $\log (n)$. Any inner product of address vectors with different addresses will be strictly less than $\log(n)$. This property facilitates our addressing operation. 

\begin{remark} [Address Property] \label{rem:address_vector_property}
In Definition~\ref{def:address_vector}, we use vectors with value $\pm 1$ to represent the address, which is different from classical $\{0, 1\}$ representation. Under this setting, we have $\forall i \in [n], a_i^\top a_i = \log (n)$, and $\forall i,j \in [n], i \neq j$, we have $a_i^\top a_j < \log(n)$ because of Cauchy–Schwarz inequality. 
\end{remark}

In this work, we mainly focus on constructing \ReLUMLP for executing ``\SUBLEQ''.
This focus stems from the fact that a One Instruction Set Computer (OISC) constructed with the ``\SUBLEQ'' instruction is functionally equivalent to a programmable computer in terms of its computational capabilities \citep{mp88}.
Since we only consider the ``\SUBLEQ'' instruction, we do not need any bits to encode the type of instruction. We only need to encode three address vectors used by the ``\SUBLEQ'' instruction. By simply concatenating the three address vectors, we have the length of the instruction as $3 \log (n)$. 

\begin{definition} [Instruction] \label{def:instruction}
Let ``\SUBLEQ'' instruction be defined as in Algorithm~\ref{alg:subleq}.
Let the address vector be defined as Definition~\ref{def:address_vector}.
Then we define the instruction vector $c_i \in \{\pm 1\}^{3 \log(n)}$ by simple concatenating three address vectors $a, b, c \in \{\pm 1 \}^{\log(n)}$ required by the ``\SUBLEQ'' instruction. Namely, we have $c_i$ satisfies the following equation:
$
    c_i = [a, b, c].
$
\end{definition}

Based on all the crucial concepts introduced above, we now introduce the state vector for our programmable computer. This state vector contains all the computer's registers, data/memory, and instructions. 

\begin{definition}[One-Bit State] \label{def:state_vector}
Let $n$ denote the size of the state vector.
Let $r_{d_1}, r_{d_1} \in \{\pm 1\}$ denote two data registers.
Let $r_c \in \{\pm 1\}$ denote the carry bit.
Let $r_{a_1}, r_{a_2}, r_{a_3} \in \{\pm 1\}^{\log (n)}$ denote three address registers.
Let $r_{pc} \in \{\pm 1\}^{\log (n)}$ denote the program counter.
Let $c_1, c_2, \cdots, c_m \in \{\pm 1\}^{3 \log (n)}$ denote $m$ instructions, where $c_m = c_{\EOF}$ is the End Of File (EOF) instruction, which means the program should terminate here.
Let $v_1, v_2, \cdots, v_k \in \{\pm 1\}$ denote $k$ one-bit data stored in the memory and the memory size $k$ satisfies $k = n - 2 - 4 \log(n) - 3 m \log(n)$. 
We define our one-bit state of \ReLUMLP as follows:
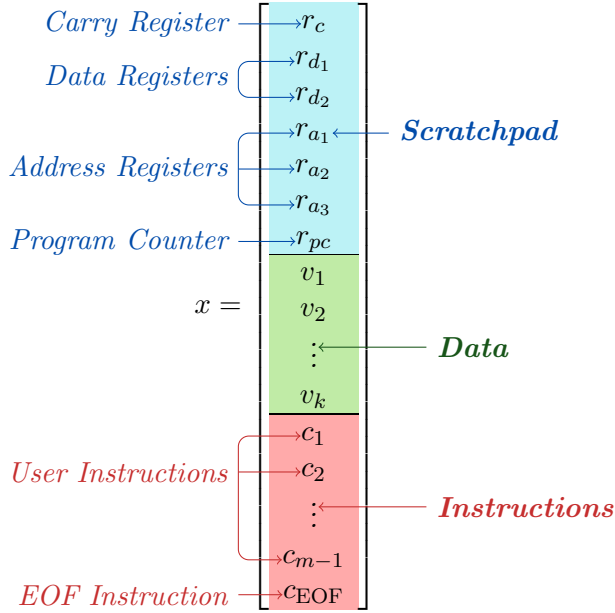
\begin{align*}
    x = 
    \begin{bmatrix}
        \begin{array}{cc}
            \bluecell \tikznode{carry_register}{$r_c$} \\
            \bluecell \tikznode{data_register_1}{$r_{d_1}$}  \\
            \bluecell \tikznode{data_register_2}{$r_{d_2}$} \\
            \bluecell \tikznode{address_register_1}{$r_{a_1}$} \\
            \bluecell \tikznode{address_register_2}{$r_{a_2}$} \\
            \bluecell \tikznode{address_register_3}{$r_{a_3}$} \\
            \bluecell \tikznode{rpc}{$r_{pc}$} \\
            \hline
            \greencell v_1 \\
            \greencell v_2 \\
            \greencell \tikznode{data}{$\vdots$} \\
            \greencell v_k \\
            \hline
            \redcell \tikznode{cmd_1}{$c_1$} \\
            \redcell \tikznode{cmd_2}{$c_2$} \\
            \redcell \tikznode{instructions}{$\vdots$} \\
            \redcell \tikznode{cmd_3}{$c_{m-1}$} \\
            \redcell \tikznode{eof}{$c_{\EOF}$} 
        \end{array}
    \end{bmatrix}
\end{align*}
\begin{tikzpicture} [remember picture,overlay,cyan,rounded corners]
    \draw[<-, color=darkblue] 
    (address_register_1) -- +(1.5,0.0)
    node[right]{\textbf{Scratchpad}};
    \draw[<-, color=darkblue] 
    (carry_register) -- +(-1.0,0.0)
    node[left]{Carry Register};
    \draw[<-, color=darkblue] 
    (data_register_1) -| +(-1.0, -0.22)
    coordinate (data_reg)
    node[left]{Data Registers};
    \draw[<-, color=darkblue] 
    (data_register_2) -| (data_reg);
    \draw[<-, color=darkblue] 
    (address_register_2) -- +(-1.0,0.0)
    coordinate (addr_reg)
    node[left]{Address Registers};
    \draw[<-, color=darkblue] 
    (address_register_1) -| (addr_reg);
    \draw[<-, color=darkblue] 
    (address_register_3) -| (addr_reg);
    \draw[<-, color=darkblue] 
    (rpc) -- +(-1.0,0.0)
    node[left]{Program Counter};
    \draw[<-, color=darkgreen] 
    (data) -- +(1.5,0.0)
    node[right]{\textbf{Data}}; 
    \draw[<-, color=darkred] 
    (instructions) -- +(1.5,0.0)
    node[right]{\textbf{Instructions}}; 
    \draw[<-, color=darkred] 
    (cmd_2) -- +(-1.0,0.0)
    coordinate (user_cmd)
    node[left]{User Instructions}; 
    \draw[<-, color=darkred] 
    (cmd_1) -| (user_cmd); 
    \draw[<-, color=darkred] 
    (cmd_3) -| (user_cmd);
    \draw[<-, color=darkred] 
    (eof) -- +(-1.0,0.0)
    node[left]{EOF Instruction};
\end{tikzpicture}
\end{definition}

\section{MAIN RESULT}\label{sec:main}

In this section, we introduce our thrilling finding, which demonstrates a looped $23$-layer \ReLUMLP is capable of emulating a programmable computer. 

\begin{theorem} [Looped \ReLUMLP as Programmable Computer, Informal Version of Theorem~\ref{thm:looped_relu_mlp_as_programmable_computer}] \label{thm:looped_relu_mlp_as_programmable_computer:informal}
Let \ReLUMLP be defined as Definition~\ref{def:relu_mlp}.
Let $n$ be the size of the state vector. 
Let $m$ be the number of instructions.
Let $k$ be the number of one-bit data stored in the memory. For $i \in [k]$, each data is $v_i \in \{ \pm 1 \}$ and the memory size $k$ satisfies $k = n - 2 - 4 \log(n) - 3 m \log(n)$.
Let the address vector $a_i \in \{\pm 1\}^{\log(n)}$.
Suppose we have two data registers $r_{d_1}, r_{d_2} \in \{\pm 1 \}$, one carry bit $r_c \in \{\pm 1 \}$, three address registers $r_{a_1}, r_{a_2}, r_{a_3} \in \{ \pm 1 \}^{\log (n)}$, and one program counter $r_{pc} \in \{ \pm 1 \}^{\log (n)}$ in the scratchpad.
Then,  a $23$-layer \ReLUMLP with width $n$ can emulate a programmable computer, where $d$ is the number of bits we use to store each integer. Namely, this ``computer'' supports integers within the range $[-2^{d-1}, 2^{d-1} - 1]$. 
\end{theorem}

The high-level idea of the proof is that we first prove the $23$-layer \ReLUMLP is capable of emulating the one-bit version ``\SUBLEQ'' instruction. Then, we extend it to supporting $d$-bits version ``\SUBLEQ''. Finally, according to the finding in \citet{mp88}, the One Instruction Set Computer (OISC) constructed by the looped \ReLUMLP is Turing complete, which indicates it is equivalent to a programmable computer. 
Due to the space limitation, we refer readers to Appendix~\ref{sec:app:looped_relu_mlp_as_programmable_computer} for more in-depth analysis. 

Furthermore, in Section~\ref{sec:discussion:low_rank}, we show that our model only requires $O(n\log n)$ number of parameters by low-rank decomposition. Thus, each forward pass of our $23$-layer \ReLUMLP only takes $O(n\log n)$ time complexity, while looped Transformer in \cite{grs+23} takes $O(n^2)$ time complexity for each forward pass. We refer readers to Section~\ref{sec:discussion:comparision_with_attention} for more details.   

\section{IMPLEMENT KEY FUNCTIONS}\label{sec:key_func}

In this section, we outline a selection of critical functions that the \ReLUMLP model is capable of executing. Due to the space limitation, the detailed proofs are deferred to the Appendix.
Specifically, in Section~\ref{sec:key_func:read}, we implement the read operation. Section~\ref{sec:key_func:write} covers the write operation. Section~\ref{sec:key_func:add} presents addition, while Section~\ref{sec:key_func:subtract} addresses subtraction. Conditional branching is implemented in Section~\ref{sec:key_func:conditional_branching}, and the \SUBLEQ~operation is in Section~\ref{sec:key_func:subleq}.

\subsection{Read}\label{sec:key_func:read}

We first consider the most basic ``read'' operation, which is to read any data or instruction from memory into a register.
We begin with introducing the read operation for one-bit data as follows:
\begin{lemma} [Read One-Bit Data, Informal Version of Lemma~\ref{lem:read_one_bit}] \label{lem:read_one_bit:informal}
Let \ReLUMLP be defined as Definition~\ref{def:relu_mlp}.
Let $n$ denote the number of data in the memory. For $i \in [n]$, each data $v_i \in \{ \pm 1 \}$.
Let the address vector $a_i \in \{\pm 1\}^{\log(n)}$.
Then, 
a $2$-layer \ReLUMLP
can read any one-bit data from the memory to the register.
\end{lemma}

Then, to achieve the read operation for $d$-bits data, we only need to use the \ReLUMLP introduced in the previous Lemma to perform a read operation on each bit in the $d$-bits data.

\begin{lemma} [Read $d$-Bits Data, informal version of Lemma~\ref{lem:read_d_bits}] \label{lem:read_d_bits:informal}
Let \ReLUMLP be defined as Definition~\ref{def:relu_mlp}.
Let $n$ denote the number of data in the memory. For $i \in [n]$, each data $v_i \in \{ \pm 1 \}$. 
Let $v \in \{\pm 1\}^d$ denote a $d$ dimension vector.
Let the address matrix $a_i \in \{\pm 1\}^{\log(n) \times d}$.
Then, 
a $2$-layer \textsf{ReLU}-\textsf{MLP} 
looped for $d$ times can read any $d$-bits data vector from the memory to the register.
\end{lemma}

\subsection{Write}\label{sec:key_func:write}

Corresponding to the read operation, in this section, we introduce how to use \ReLUMLP to implement the ``write'' operation. 
We start with the write operation of one-bit data.
\begin{lemma} [Write One-Bit Data, Informal Version of Lemma~\ref{lem:write_one_bit}] \label{lem:write_one_bit:informal}
Let \ReLUMLP be defined as Definition~\ref{def:relu_mlp}.
Let $n$ denote the number of data in the memory. For $i \in [n]$, each data $v_i \in \{ \pm 1 \}$.
Let the address vector $a_i \in \{\pm 1\}^{\log(n)}$.
Then, 
a $2$-layer \ReLUMLP
can write any one-bit data from the register to the memory. 
\end{lemma}

Subsequently, we can expand our approach to accommodate the ``write'' operation for $d$-bit data by sequentially processing each bit within the $d$-bit data.
\begin{lemma} [Write $d$-Bits Data, Informal Version of Lemma~\ref{lem:write_d_bits}] \label{lem:write_d_bits:informal}
Let \ReLUMLP be defined as Definition~\ref{def:relu_mlp}.
Let $n$ denote the number of data in the memory. For $i \in [n]$, each data $v_i \in \{ \pm 1 \}$. 
Let $v \in \{\pm 1\}^d$ denote a $d$ dimension vector.
Let the address matrix $a_i \in \{\pm 1\}^{\log(n) \times d}$.
Then, a $2$-layer \textsf{ReLU}-\textsf{MLP}
looped for $d$ times can write any $d$-bits data vector from the register to the memory.
\end{lemma}

\subsection{Addition}\label{sec:key_func:add}

Beyond the fundamental memory operations of reading and writing, a pivotal set of operations involves algorithmic functions. 
Consequently, we demonstrate how the addition and subtraction operations can be emulated by the \textsf{ReLU}-\textsf{MLP}. We begin with introducing the emulation of one-bit addition operation via \textsf{ReLU-MLP}. 

\begin{lemma} [One-Bit Addition, Informal Version of Lemma~\ref{lem:one_bit_addition}] \label{lem:one_bit_addition:informal}
Let \ReLUMLP be defined as Definition~\ref{def:relu_mlp}.
Then, a $6$-layer \ReLUMLP can emulate the ``addition'' operation for any one-bit data. 
\end{lemma}

We extend to $d$-bits addition as follows:
\begin{lemma} [$d$-Bits Addition, Informal Version of Lemma~\ref{lem:d_bit_addition}] \label{lem:d_bit_addition:informal}
Let \ReLUMLP be defined as Definition~\ref{def:relu_mlp}.
Then, a $6$-layer \textsf{ReLU}-\textsf{MLP} looped for $2d$ times (due to carry bit) can emulate the ``addition'' operation for any $d$-dimension vectors. 
\end{lemma}

\subsection{Subtraction} \label{sec:key_func:subtract}

We move on to introducing the subtraction operation. Since we use $2$'s complement (Definition~\ref{def:complement}) as the representation of the data, the subtraction operation can be decomposed into two steps. The first step is to negate the subtrahend, and the second step is to add the negated subtrahend to the minuend. 
Combining the above statement with the addition operation introduced in the previous section, we can easily perform a subtraction operation with a $7$-layer \textsf{ReLU-MLP}. 

\begin{lemma}[$d$-Bits Subtraction, Informal Version of Lemma~\ref{lem:subtraction}] \label{lem:subtraction:informal}
Let \ReLUMLP be defined as Definition~\ref{def:relu_mlp}.
Then, a $7$-layer \ReLUMLP can emulate the ``subtraction'' operation for any $d$-dimension vectors. 
\end{lemma}

\subsection{Conditional Branching}\label{sec:key_func:conditional_branching}

Conditional branching is an essential operation in computer design since it enables the processor to make decisions and execute different paths of code based on the evaluation of conditions (the flag), thereby allowing for more complex and adaptable program flow.
We present our design for conditional branching via \ReLUMLP as follows:
\begin{lemma}[Conditional Branching, Informal Version of Lemma~\ref{lem:conditional_branching}] \label{lem:conditional_branching:informal}
Let \ReLUMLP be defined as Definition~\ref{def:relu_mlp}.
Let $n$ be the number of data in the memory. For $i \in [n]$, each data $v_i \in \{ \pm 1 \}$.
Let the address vector $a_i \in \{\pm 1\}^{\log(n)}$.
Then, a $4$-layer \ReLUMLP can emulate the ``conditional branching'' operation.
\end{lemma}

\subsection{\texorpdfstring{\SUBLEQ}{} Instruction}\label{sec:key_func:subleq}

Now, we present our method for implementing the ``\SUBLEQ'' instruction using a looped $23$-layer \textsf{ReLU-MLP}. The ``\SUBLEQ'' instruction operates with three addresses as parameters, denoted as $a$, $b$, and $c$. It first computes the subtraction between the values stored at addresses $b$ and $a$, i.e., $\mem[b]-\mem[a]$, and then updates the memory at $\mem[b]$ with this result. If the value at $\mem[b]$ is zero or negative, the program counter will be transferred to the address specified by $c$; otherwise, the program counter will go to the next instruction without branching.
Algorithm~\ref{alg:subleq} demonstrates the execution process of ``\SUBLEQ'' instruction. 

\begin{algorithm}[!ht]
\caption{SUBtract and branch if
Less-than or Equal to zero (\SUBLEQ), \citep{mp88,grs+23}}\label{alg:subleq}
\begin{algorithmic}[1]
\Procedure{\SUBLEQ}{$a$, $b$, $c$}
    \State {\texttt{mem}[$b$] = 
    \texttt{mem}[$b$] - \texttt{mem}[$a$]}  
    \Comment{$\texttt{mem}[a]$ denotes the value of address $a$ in the memory}
    \If {\texttt{mem}[$b$] $\leq$ $0$} \State{\texttt{goto} instruction $c$}
    \Else { \texttt{goto} next instruction}
    \EndIf
\EndProcedure
\end{algorithmic}
\end{algorithm}

We integrate the read and write operations, along with the addition and subtraction operations and the conditional branching mechanisms discussed in the preceding sections, to facilitate the implementation of the ``\SUBLEQ'' instruction. The accompanying lemma is stated as follows:
\begin{lemma}[\ReLUMLP Emulate \SUBLEQ, Informal Version of Lemma~\ref{lem:subleq}]
\label{lem:subleq:informal}
Let \ReLUMLP be defined as Definition~\ref{def:relu_mlp}.
Let $n$ denote the size of the state vector.
Let $m$ denote the number of instructions.
Let $k$ denote the number of one-bit data stored in the memory. For $i \in [k]$, each data is $v_i \in \{ \pm 1 \}$ and the memory size $k$ satisfies $k = n - 2 - 4 \log(n) - 3 m \log(n)$.
Let the address vector $a_i \in \{\pm 1\}^{\log(n)}$.
Let the instruction $c_i \in \{ \pm 1 \}^{3 \log (n)}$ be defined as Definition~\ref{def:instruction}.
Suppose we have three data registers $r_c, r_{d_1}, r_{d_2} \in \{\pm 1 \}$, one carry bit $r_c \in \{\pm 1 \}$, three address registers $r_{a_1}, r_{a_2}, r_{a_3} \in \{ \pm 1 \}^{\log (n)}$, and one program counter $r_{pc} \in \{ \pm 1 \}^{\log (n)}$ in the scratchpad.
Then, a $23$-layer \ReLUMLP with width $n$ can emulate the ``\SUBLEQ'' operation (Algorithm~\ref{alg:subleq}).
\end{lemma}

To make it easier for readers to understand our proof, we provide a proof sketch here, which contains some high-level ideas used in our proof.
Specifically, we first read the necessary address vectors and data required by the instruction from the memory. Then, we perform subtraction and conditional branching via the respective \textsf{ReLU-MLP}s discussed in the previous sections. 
\begin{proof}[Proof sketch]
We use the state vector $x$ as defined in Definition~\ref{def:state_vector}. 
The first step is to read the ``\SUBLEQ'' instruction and the data required by the instruction from memory, where the read operation is supported by the \ReLUMLP discussed in Lemma~\ref{lem:read_d_bits:informal}. After this step, the state vector will change as follows:
\begin{align*}
    x =
    \begin{bmatrix}
        \begin{array}{cc}
            r_c \\
            r_{d_1} \\
            r_{d_2} \\
            \colorbox{lightblue}{$r_{a_1}$} \\
            \colorbox{lightgreen}{$r_{a_2}$} \\
            \colorbox{lightred}{$r_{a_3}$} \\
            r_{pc} \\
            \hline
            v_1 \\
            \vdots \\
            v_k \\
            \hline 
            c_1 \\
            \vdots \\
            c_{m-1} \\
            c_{\EOF}
        \end{array}
    \end{bmatrix}
    \rightarrow
    \begin{bmatrix}
        \begin{array}{cc}
            r_c \\
            \colorbox{lightpurple}{$r_{d_1}$} \\
            \colorbox{lightyellow}{$r_{d_2}$} \\
            \colorbox{lightblue}{$a$} \\
            \colorbox{lightgreen}{$b$} \\
            \colorbox{lightred}{$c$} \\
            r_{pc} \\
            \hline
            v_1 \\
            \vdots \\
            v_k \\
            \hline 
            c_1 \\
            \vdots \\
            c_{m-1} \\
            c_{\EOF}
        \end{array}
    \end{bmatrix}
    \rightarrow
    \begin{bmatrix}
        \begin{array}{cc}
            r_c \\
            \colorbox{lightpurple}{$\mem[a]$} \\
            \colorbox{lightyellow}{$\mem[b]$} \\
            a \\
            b \\
            c \\
            r_{pc} \\
            \hline
            v_1 \\
            \vdots \\
            v_k \\
            \hline 
            c_1 \\
            \vdots \\
            c_{m-1} \\
            c_{\EOF}
        \end{array}
    \end{bmatrix}
\end{align*}
where we first read the address vectors $a$ and $b$ from the memory to the address registers, then read the corresponding data $\mem[a]$ and $\mem[b]$ from the memory to the data registers. 

In the second step, we calculate the subtraction of $\mem[b]$ and $\mem[a]$, store its result to $\mem[b]$, and calculate the $r_{pc+1}$ according to $r_{pc}$, storing $r_{pc+1}$ at the address register which previously stores $b$. The above operations can be supported by Lemma~\ref{lem:d_bit_addition:informal} and Lemma~\ref{lem:write_d_bits:informal}. After this step, the state vector will change as follows:
\begin{align*}
    x = 
    \begin{bmatrix}
        \begin{array}{cc}
            r_c \\
            \colorbox{lightblue}{$\mem[a]$} \\
            \mem[b] \\
            a \\
            \colorbox{lightgreen}{$b$} \\
            c \\
            r_{pc} \\
            \hline
            v_1 \\
            \vdots \\
            v_k \\
            \hline 
            c_1 \\
            \vdots \\
            c_{m-1} \\
            c_{\EOF}
        \end{array}
    \end{bmatrix}
    \rightarrow
    \begin{bmatrix}
        \begin{array}{cc}
            r_c \\
            \colorbox{lightblue}{$\mem[b] - \mem[a]$} \\
            \mem[b] \\
            a \\
            \colorbox{lightgreen}{$r_{pc+1}$} \\
            c \\
            r_{pc} \\
            \hline
            v_1 \\
            \vdots \\
            v_k \\
            \hline 
            c_1 \\
            \vdots \\
            c_{m-1} \\
            c_{\EOF}
        \end{array}
    \end{bmatrix}
\end{align*}
In the final step, we perform the conditional branching operation. We view the result $\mem[b] - \mem[a]$ as the flag of the conditional branching. Then, by Lemma~\ref{lem:conditional_branching:informal}, the program counter will point to the target address $r_{\target}$ for continue executing the program, where we have $r_{\target} \in \{c, r_{pc+1}\}$. After this step, the state vector will change as follows:
\begin{align*}
    x = 
    \begin{bmatrix}
        \begin{array}{cc}
            r_c \\
            \colorbox{lightblue}{$\mem[b] - \mem[a]$} \\
            \mem[b] \\
            a \\
            r_{pc+1} \\
            c \\
            \colorbox{lightgreen}{$r_{pc}$} \\
            \hline
            v_1 \\
            \vdots \\
            v_{b-1} \\
            \colorbox{lightred}{$v_b$} \\
            v_{b+1} \\
            \vdots \\
            v_k \\
            \hline 
            c_1 \\
            \vdots \\
            c_{m-1} \\
            c_{\EOF}
        \end{array}
    \end{bmatrix}
    \rightarrow
    \begin{bmatrix}
        \begin{array}{cc}
            r_c \\
            \colorbox{lightblue}{$\flag$} \\
            \mem[b] \\
            a \\
            r_{pc+1} \\
            c \\
            \colorbox{lightgreen}{$r_{\target}$} \\
            \hline
            v_1 \\
            \vdots \\
            v_{b-1} \\
            \colorbox{lightred}{$\mem[b] - \mem[a]$} \\
            v_{b+1} \\
            \vdots \\
            v_k \\
            \hline 
            c_1 \\
            \vdots \\
            c_{m-1} \\
            c_{\EOF}
        \end{array}
    \end{bmatrix}
\end{align*}
In the first step, we treat the value of $\mem[b] - \mem[a]$ merely as a flag for conditional branching purposes. No operations are conducted in this first step. 
In the second step, we perform the conditional branching based on the flag's value. Specifically, if the flag is zero or negative, the target register $r_{\target}$ is set to the value of $c$. Conversely, if the flag is positive, $r_{\target}$ is assigned the value of $r_{pc+1}$.

Up to this point, we have obtained the execution result for the current ``\SUBLEQ'' instruction. To process the subsequent instruction, we simply need to iterate through the provided $23$-layer \ReLUMLP once more, thereby executing the ``\SUBLEQ'' instruction for the next cycle.
\end{proof}

We only provided a sketch of the proof above. We refer the readers to Appendix~\ref{sec:app:subleq} for more details. 

\section{DISCUSSION AND EXTENSION} \label{sec:discuss}
Section~\ref{sec:discussion:low_rank} shows that we can reduce the number of parameters. 
Section~\ref{sec:discussion:comparision_with_attention} compares \ReLUMLP with the attention mechanism.
In Section~\ref{sec:discussion:exploring_relumlp_potential}, we discuss the potential capability of \textsf{ReLU-MLP}.
In Section~\ref{sec:discussion:towards_more_efficent_design}, we discuss the potential inspirations our work offers for designing more efficient neural network architectures.

\subsection{Low-rank Decomposition for Efficiency}\label{sec:discussion:low_rank}

Though the naive way to construct the looped-\textsf{ReLU-MLP} will take up to $O(n^2)$ parameter, the parameter requirement can be easily reduced by low-rank decomposition. 
The high-level idea is that, for all $n \times n$ matrices used in our proof are equal to an identity matrix minus a matrix with rank $O(\log(n))$. Thus, we can decompose our weight matrix into a low-rank format from $n \times n$ into $n \times O(\log(n))$ and $O(\log(n)) \times n$.

We formalize the above high-level idea into math formulas. Let $x \in \mathbb{R}^n$ be a vector, where $n$ is the sequence length, and let $W$ be the weight matrix.
As all our matrices can be viewed as an identity matrix with a substitution of a submatrix with some $O(\log n) \times O(\log n)$ matrix on the diagonal, i.e., the matrix $W$ is equal to an identity matrix minus a matrix $A$ with rank $O(\log(n))$, corresponding to residual connections. Then, we have $A$ can be decomposed into $A = U_A V_A^\top$, where $U_A, V_A \in \mathbb{R}^{n \times O(\log n)}$. And we have $Wx = (I - U_A V_A^\top)x = x - U_A V_A^\top x$. Since we only need to store matrices $U_A$ and $V_A$, the parameter is reduced from $O(n^2)$ to $O(n \log n)$.
In other words, for each operation, we only need the $O(\log n) \times O(\log n)$ matrix to modify the corresponding part of the state vector (vector with length $n$) and keep the rest part of the state vector unchanged.
For example, the $W_1$ matrix used in Lemma C.1 is a submatrix. Then, the entire matrix is $W'_1 \in \mathbb{R}^{n \times n}$, where
\begin{align*}
    W'_1 
    = & ~ \mathbf{I} - 
    \begin{bmatrix} 
    W_1 + I & 0 & 0 & \cdots & 0 \\
    0 & 0 & 0 & \cdots & 0 \\
    0 & 0 & 0 & \cdots & 0 \\
    \vdots & \vdots & \vdots & \ddots & \vdots \\
    0 & 0 & 0 & \cdots & 0
    \end{bmatrix} = \mathbf{I} + A
\end{align*}
Since $W_1 - I \in \mathbb{R}^{O(\log n) \times O(\log n)}$, the matrix $A$ is at most rank-$O(\log n)$. Therefore, when equipped with low-rank decomposition methods, all $n \times n$ matrices used in our method can be decomposed into $n \times \log(n)$ matrices. The overall parameter used by our method is $O(n \log (n))$.

\subsection{Comparison with Attention} \label{sec:discussion:comparision_with_attention}
In terms of computational cost, since all matrices in our method can be decomposed into matrices with rank-$O(\log n)$, the computation cost for one forward pass is $O(n \log n)$. This is because the computational bottleneck of our method is the multiplication of the $n \times \log n$ size matrix and the $n$ size state vector, which requires $O(n \log n)$ time. However, for the looped Transformers proposed in \cite{grs+23}, they have to calculate the matrix multiplication of the $n \times n$ attention matrix and the $n \times d$ value matrix in each forward pass, which requires $O(n^2)$ time complexity. 

On the other hand, as discussed in Section~\ref{sec:main}, we have demonstrated that a looped $23$-layer \ReLUMLP is capable of emulating a programmable computer. In contrast to the findings reported by \cite{grs+23}, where a looped 9-layer Transformer is shown to be capable of such emulation, our result indicates that even a basic component of deep learning, the \textsf{ReLU-MLP}, possesses the potential to handle complex computational tasks. This suggests that while advanced architectures like Transformers have shown proficiency in processing intricate tasks, this proficiency may not come from its advanced architecture design. Instead, it may derive from the inherent capabilities of fundamental components such as the \textsf{ReLU-MLP}. Our finding underscores the importance of investigating the core mechanisms behind the capabilities of advanced architectures, an area that warrants further exploration.

\subsection{Exploring the Potential of \textsf{ReLU-MLP}} \label{sec:discussion:exploring_relumlp_potential}
Our research has confirmed that the capabilities of the \textsf{ReLU-MLP} are on par with Transformers when it comes to constructing programmable computers.
As noted in \cite{p99, kl20}, \textsf{ReLU-MLP}s have been demonstrated to be universal approximators. Consequently, we conjecture the programmable computer represents just one of the many downstream tasks that \textsf{ReLU-MLP}s are capable of achieving. Building on the insights from this study, it is promising to further investigate the potential capabilities of \textsf{ReLU-MLP}s.
Our approach to analyzing the looped \textsf{ReLU-MLP} mitigates the black-box nature often associated with traditional deep learning training.
We are confident that through the analytical methods outlined in our work, we can systematically probe the capacity of \textsf{ReLU-MLP}s as universal approximators to tackle more complex tasks. This line of work will be done in our future research.

\subsection{Towards More Efficient Architecture Design for Specific Tasks} \label{sec:discussion:towards_more_efficent_design}
The construction-based proof in our work can inspire future explorations of the smallest feasible network structure for specific tasks. 
Since the main focus of our work is to prove the feasibility of \textsf{ReLU-MLP}-based computer construction, the \textsf{ReLU-MLP}-based construction we provide may not be the smallest construction that can support a programmable computer. Therefore, in fact, we provide a possible lower bound that can achieve a programmable computer. The current mainstream methods for model compression are mainly quantization and model pruning, which often provide model compression solutions based on experimental results and lack theoretical measurements between model capabilities and model capabilities required by the downstream tasks. 
Thus, leveraging the methodologies applied in this study, we can first identify the minimal requirements of a model for a particular task, enabling researchers to create the smallest yet efficient models to handle it. This strategy permits us to bypass the costly process of the expensive process of scaling up data or model parameters. 
We leave these directions as our future directions. 

\section{CONCLUSION} \label{sec:conclusion}
In this work, we investigate the computational potential of a looped \textsf{ReLU-MLP}. We begin by demonstrating that a looped $23$-layer \textsf{ReLU-MLP} with a single pass is capable of emulating the execution process of the ``\SUBLEQ'' instruction. Drawing on the conclusions presented in \cite{mp88}, we establish that the One Instruction Set Computer (OISC) implemented by the looped $23$-layer \ReLUMLP is functionally equivalent to a programmable computer.
Contrary to \cite{grs+23}, which achieved the emulation of a programmable computer using a looped 9-layer Transformer, our approach leverages a more efficient and more fundamental building block of deep learning, a $23$-layer \textsf{ReLU-MLP}, to accomplish the same objective. 
This finding prompts us to consider two key insights: ($i$) The untapped potential of \textsf{ReLU-MLP}s warrants further exploration, offering a deeper understanding of the capability boundaries of existing neural networks; ($ii$) By adopting the methodologies employed in this study, it may facilitate us to gain insights for designing the minimal neural network requirements for any specific tasks. 

\ifdefined\isarxiv
\section*{Acknowledgement}
Research is partially supported by the National Science Foundation (NSF) Grants 2023239-DMS, CCF-2046710, and Air Force Grant FA9550-18-1-0166.
\bibliographystyle{alpha}
\bibliography{ref}
\else

\section*{Acknowledgement}
Research is partially supported by the National Science Foundation (NSF) Grants 2023239-DMS, CCF-2046710, and Air Force Grant FA9550-18-1-0166.
\bibliography{ref}
\bibliographystyle{plainnat}
\input{checklist}

\fi

\newpage
\onecolumn
\appendix

\ifdefined\isarxiv

\begin{center}
    \textbf{\LARGE Appendix }
\end{center}

\else

\aistatstitle{
Looped ReLU MLPs May Be All You Need \\as Programmable Computers: \\Supplementary Materials}

{\hypersetup{linkcolor=black}
\tableofcontents
\bigbreak
}
\fi

\paragraph{Roadmap.} 
We organize our appendix as follows:
In Section~\ref{sec:app:read_and_write}, we introduce our construction of \ReLUMLP for emulating the read and write operation.
In Section~\ref{sec:app:addt_and_subtract}, we present the way we achieve the addition and subtraction operation via \ReLUMLP.
In Section~\ref{sec:app:conditional_branching}, we show that a $4$-layer \ReLUMLP is capable of emulating the conditional branching operation.
In Section~\ref{sec:app:subleq}, we illustrate the ``\SUBLEQ'' can be emulated by a looped $23$-layer \textsf{ReLU-MLP}. 
In Section~\ref{sec:app:looped_relu_mlp_as_programmable_computer}, we demonstrate how the looped $23$-layer \ReLUMLP introduced in the previous section is capable of functioning as a programmable computer.

\section{READ AND WRITE} \label{sec:app:read_and_write}

In this section, we mainly focus on the construction of a multi-layer \ReLUMLP for supporting read and write operations.
We begin with the scenario where we read one-bit data from the memory to the data register.

\begin{lemma} [Read One-Bit Data, Formal Version of Lemma~\ref{lem:read_one_bit:informal}] \label{lem:read_one_bit}
If the following conditions hold:
\begin{itemize}
    \item Let \ReLUMLP be defined as Definition~\ref{def:relu_mlp}.
    \item Let $n$ denote the number of data in the memory. For $i \in [n]$, each data $v_i \in \{ \pm 1 \}$. 
    \item Let the address vector $a_i \in \{\pm 1\}^{\log(n)}$.
\end{itemize}

Then, we can show that a $2$-layer \ReLUMLP
can read any one-bit data from the memory to the register.
\end{lemma}

\begin{proof}
Consider a simplified case where we have one 
data register $r_d \in \{\pm 1\}$ which stores data $v_0 \in \{\pm 1\}$ initially, and one address register $r_a \in \{\pm 1\}^{\log(n)}$ which stores address $a_i \in \{\pm 1\}^{\log(n)}$ initially. The address $a_i$ points to the location where the data should be copied from. 
Namely, we want to perform the following operation:
\begin{align*}
    x := 
    \begin{bmatrix}
        \begin{array}{c}
            \colorbox{lightblue}{$r_d:v_0$} \\
            r_a:a_i \\
            \hline 
            v_1 \\
            v_2 \\
            \vdots \\
            v_n \\
        \end{array}
    \end{bmatrix}
    \rightarrow
    \begin{bmatrix}
        \begin{array}{c}
            \colorbox{lightblue}{$r_d:v_i$} \\
            r_a:a_i \\
            \hline 
            v_1 \\
            v_2 \\
            \vdots \\
            v_n \\
        \end{array}
    \end{bmatrix}
\end{align*}

For simplicity, we ignore the notations $r_d:$ and $r_a:$ in the following proof. We use $\0_{d} \in \R^d$ to denote a vector with entries are $0$.

{\bf Step 1: Erase $v_0$.}

We construct the following weight matrix $W_1 \in \R^{(n+\log(n)+1) \times (n+\log(n)+1)}$:
\begin{align*}
    W_1 := 
    \begin{bmatrix}
        0 & \0_{\log (n)}^\top & 0 & \cdots & 0 \\
        \0_{\log (n)} & I_{\log (n) \times \log (n)} & \0_{\log (n)} & \cdots & \0_{\log (n)} \\
        0 & \0_{\log (n)}^\top & 1 & \cdots & 0 \\
        \vdots & \vdots & \vdots & \ddots & \vdots \\
        0 & \0_{\log (n)}^\top & 0 & \cdots & 1 \\
    \end{bmatrix}
\end{align*}

Then we erase $v_0$ by $W_1$.
\begin{align*}
    x_1 := W_1 x =
    \begin{bmatrix}
        \begin{array}{c}
            0 \\
            a_i \\
            \hline 
            v_1 \\
            v_2 \\
            \vdots \\
            v_n \\
        \end{array}
    \end{bmatrix}
\end{align*}

{\bf Step 2: Extract $v_i$.}

We need to construct the location vector first. We construct the following weight matrix $W_2 \in \R^{n \times \log(n)}$:
\begin{align*}
    W_2 := [a_1, \cdots, a_n]^\top
\end{align*}

Then, we perform the following operations to get the location vector:
\begin{align*}
    \ReLU(
    \underbrace{
    W_2
    }_{n \times \log (n)}
    \underbrace{a_i}_{\log(n) \times 1}
    - \underbrace{(\log (n) - 1)}_{1 \times 1} \cdot \underbrace{\1_n}_{n \times 1}
    )
    = \underbrace{e_i}_{n \times 1}
\end{align*}
where the first step follows from we have $\forall i, a_i^\top a_i = \log (n)$ and $\forall i \neq j$, $a_i^\top a_j < \log(n)$ (Remark~\ref{rem:address_vector_property}).
Here, the $e_i$ denotes the vector whose $i$-th entry is $1$, and other entries are $0$. 

Then, we extract $v_i$ by the following operation:
\begin{align*}
    \underbrace{e_i^\top}_{1 \times n}
    \underbrace{
    \begin{bmatrix}
        v_1 \\
        v_2 \\
        \vdots \\
        v_n
    \end{bmatrix}
    }_{n \times 1}
    = \underbrace{v_i}_{1 \times 1}
\end{align*}

{\bf Step 3: Put $v_i$ to register.}

Then, we construct $x_{v_i}$ by concatenating $v_i$ with some zero vectors:
\begin{align*}
    x_{v_i} := 
    \begin{bmatrix}
        \begin{array}{c}
            v_i \\
            \0_{\log (n)} \\
            \hline 
            0 \\
            0 \\
            \vdots \\
            0 \\
        \end{array}
    \end{bmatrix}
\end{align*}

Finally, we add $x_{v_i}$ and $x_1$ together to get our final result.
\begin{align*}
    y := x_1 + x_{v_i} = 
    \begin{bmatrix}
        \begin{array}{c}
            v_i \\
            a_i \\
            \hline 
            v_1 \\
            v_2 \\
            \vdots \\
            v_n \\
        \end{array}
    \end{bmatrix}
\end{align*}

To sum up, we have
\begin{itemize}
    \item {\bf Step 1} uses one-layer \textsf{ReLU}-\textsf{MLP}.
    \item {\bf Step 2} uses one-layer \textsf{ReLU}-\textsf{MLP}.
    \item {\bf Step 3} uses vector addition, so doesn't use \textsf{ReLU}-\textsf{MLP}.
\end{itemize}

Therefore, we use a two-layer \textsf{ReLU}-\textsf{MLP} to emulate the ``read'' operation.
\end{proof}

Then, we move on to considering the read operation for $d$-bits data.

\begin{lemma} [Read $d$-Bits Data, Formal Version of Lemma~\ref{lem:read_d_bits:informal}] \label{lem:read_d_bits}
If the following conditions hold:
\begin{itemize}
    \item Let \ReLUMLP be defined as Definition~\ref{def:relu_mlp}.
    \item Let $n$ denote the number of data in the memory. For $i \in [n]$, each data $v_i \in \{ \pm 1 \}$. 
    \item Let $v \in \{\pm 1\}^d$ denote a $d$ dimension vector.
    \item Let the address matrix $a_i \in \{\pm 1\}^{\log(n) \times d}$.
\end{itemize}

Then, we can show that a $2$-layer \textsf{ReLU}-\textsf{MLP}, 
looped for $d$ times, can read any $d$-bits data vector from the memory to the register.
\end{lemma}

\begin{proof}
By Lemma~\ref{lem:read_one_bit}, we can read one bit from memory to register via a two-layer \textsf{ReLU}-\textsf{MLP}.

Considering the case where we have an input matrix $X$ with size $(1 + \log (n) + n) \times d$, which can be written as follows, where the superscript denotes the dimension:
\begin{align*}
    X :=
    \begin{bmatrix}
        \begin{array}{cccc}
            r_d^1 & r_d^2 & \cdots & r_d^d\\
            r_a^1 & r_a^2 &\cdots & r_a^d\\
            \hline 
            v_1^1 & v_1^2 &\cdots & v_1^d\\
            v_2^1 & v_2^2 &\cdots & v_2^d\\
            \vdots & \vdots & \ddots & \vdots \\
            v_n^1 & v_n^2 &\cdots & v_n^d\\
        \end{array}
    \end{bmatrix}
\end{align*}

We apply the two-layer \ReLUMLP to each column of $X$. Since $X$ has $d$ columns, we loop the two layers \ReLUMLP for $d$ times. Then, we can perform the ``read'' operation for any $d$-dimension vector from the memory to the register. 
\end{proof}

Writing can be considered as the inverse process of reading.
Summarily, we first consider writing one-bit data from the data register to the memory. 

\begin{lemma} [Write One-Bit Data, Formal Version of Lemma~\ref{lem:write_one_bit:informal}] \label{lem:write_one_bit}
If the following conditions hold:
\begin{itemize}
    \item Let \ReLUMLP be defined as Definition~\ref{def:relu_mlp}.
    \item Let $n$ denote the number of data in the memory. For $i \in [n]$, each data $v_i \in \{ \pm 1 \}$. 
    \item Let the address vector $a_i \in \{\pm 1\}^{\log(n)}$. 
\end{itemize}

Then, we can show that a $2$-layer \ReLUMLP
can write any one-bit data from the register to the memory. 
\end{lemma}

\begin{proof}
Consider a simplified case where we have one 
data register $r_d \in \{\pm 1\}$ which stores data $v_0 \in \{\pm 1\}$ initially, and one address register $r_a \in \{\pm 1\}^{\log (n)}$ which stores address $a_i \in \in \{\pm 1\}^{\log (n)}$ initially. The address $a_i$ points to the location where the data should be copied from. 
Namely, we want to perform the following operation:
\begin{align*}
    x := 
    \begin{bmatrix}
        \begin{array}{c}
            r_d:v_0 \\
            r_a:a_i \\
            \hline 
            v_1 \\
            v_2 \\
            \vdots \\
            v_{i-1} \\
            \colorbox{lightblue}{$v_i$} \\
            v_{i+1} \\
            \vdots \\
            v_n \\
        \end{array}
    \end{bmatrix}
    \rightarrow
    \begin{bmatrix}
        \begin{array}{c}
            r_d:v_0 \\
            r_a:a_i \\
            \hline 
            v_1 \\
            v_2 \\
            \vdots \\
            v_{i-1} \\
            \colorbox{lightblue}{$v_0$} \\
            v_{i+1} \\
            \vdots \\
            v_n \\
        \end{array}
    \end{bmatrix}
\end{align*}

For simplicity, we ignore the notations $r_d:$ and $r_a:$ in the following proof.

Let $N := n + \log(n) + 1$. Then we have $x \in \R^N$.

{\bf Step 1: Erase $v_i$.}

We need to construct the location vector first. We construct the following weight matrix $W_1 \in \R^{n \times \log(n)}$:
\begin{align*}
    W_1 := [a_1, \cdots, a_n]^\top
\end{align*}

Then, we perform the following operations:
\begin{align*}
    \ReLU(
    \underbrace{
    W_1
    }_{n \times \log (n)} \cdot
    \underbrace{a_i}_{\log(n) \times 1}
    - \underbrace{(\log (n) - 1)}_{1 \times 1} \cdot \underbrace{\1_n}_{n \times 1}
    )
    = \underbrace{e_i}_{n \times 1}
\end{align*}
where the first step follows from we have $\forall i, a_i^\top a_i = \log (n)$ and $\forall i \neq j$, $a_i^\top a_j < \log(n)$ (Remark~\ref{rem:address_vector_property}).
Here, the $e_i$ denotes the vector whose $i$-th entry is $1$, and other entries are $0$.

Then, we erase the $v_i$ in $x$ by first performing the following operation:
\begin{align*}
    \underbrace{(1 - e_i)}_{n \times 1} \circ 
    \underbrace{
    \begin{bmatrix}
        v_1 \\
        v_2 \\
        \vdots \\
        v_{i-1} \\
        v_i \\
        v_{i+1} \\
        \vdots \\
        v_n
    \end{bmatrix}
    }_{n \times 1} = 
    \underbrace{
    \begin{bmatrix}
        v_1 \\
        v_2 \\
        \vdots \\
        v_{i-1} \\
        0 \\
        v_{i+1} \\
        \vdots \\
        v_n
    \end{bmatrix}
    }_{n \times 1}
\end{align*}

Then, we define $x_1$ as follows:
\begin{align*}
    x_1 := 
    \begin{bmatrix}
        \begin{array}{c}
            v_0 \\
            a_i \\
            \hline 
            v_1 \\
            v_2 \\
            \vdots \\
            v_{i-1} \\
            0 \\
            v_{i+1} \\
            \vdots \\
            v_n \\
        \end{array}
    \end{bmatrix}
\end{align*}

{\bf Step 2: Extract $v_0$.}

We construct the following weight matrix $W_2 \in \R^{1 \times N}$:
\begin{align*}
    W_2 := [1, \0_{\log (n)}^\top, \0_n^\top]
\end{align*}

Then, we have
\begin{align*}
    \underbrace{W_2}_{1 \times N} \cdot \underbrace{x_1}_{N \times 1}  = \underbrace{v_0}_{1 \times 1}
\end{align*}

{\bf Step 3: Put $v_0$ to memory.}

We use the $e_i$ acquired from {\bf Step 1} to perform the following operation:
\begin{align*}
    \underbrace{v_0}_{1 \times 1} \cdot \underbrace{e_i}_{n \times 1} + 
    \underbrace{
    \begin{bmatrix}
        v_1 \\
        v_2 \\
        \vdots \\
        v_{i-1} \\
        0 \\
        v_{i+1} \\
        \vdots \\
        v_n
    \end{bmatrix}
    }_{n \times 1}
    = 
    \underbrace{
    \begin{bmatrix}
        v_1 \\
        v_2 \\
        \vdots \\
        v_{i-1} \\
        v_0 \\
        v_{i+1} \\
        \vdots \\
        v_n
    \end{bmatrix}
    }_{n \times 1}
\end{align*}

Then, by concatenation, we have our final result:
\begin{align*}
    y := 
    \begin{bmatrix}
        \begin{array}{c}
            r_d:v_0 \\
            r_a:a_i \\
            \hline 
            v_1 \\
            v_2 \\
            \vdots \\
            v_{i-1} \\
            v_0 \\
            v_{i+1} \\
            \vdots \\
            v_n \\
        \end{array}
    \end{bmatrix}
\end{align*}

To sum up, we have
\begin{itemize}
    \item {\bf Step 1} uses one-layer \textsf{ReLU}-\textsf{MLP}.
    \item {\bf Step 2} 
    uses one-layer \textsf{ReLU}-\textsf{MLP}.
    \item {\bf Step 3} use concatenation, so doesn't use \textsf{ReLU}-\textsf{MLP}.
\end{itemize}

Therefore, we use a two-layer \ReLUMLP to emulate the ``write'' operation.
\end{proof}

Then, we show how we extend writing one-bit data to writing $d$-bits data as follows:
\begin{lemma} [Write $d$-Bits Data, Formal Version of Lemma~\ref{lem:write_d_bits:informal}] \label{lem:write_d_bits}
If the following conditions hold:
\begin{itemize}
    \item Let \ReLUMLP be defined as Definition~\ref{def:relu_mlp}.
    \item Let $n$ denote the number of data in the memory. For $i \in [n]$, each data $v_i \in \{ \pm 1 \}$. 
    \item Let $v \in \{\pm 1\}^d$ denote a $d$ dimension vector.
    \item Let the address matrix $a_i \in \{\pm 1\}^{\log(n) \times d}$.
\end{itemize}

Then, we can show that, 
a $2$-layer \textsf{ReLU}-\textsf{MLP}, looped for $d$ times, can write any $d$-bits data vector from the register to the memory.
\end{lemma}

\begin{proof}
By Lemma~\ref{lem:write_one_bit}, we can write one bit from the register to the memory via a two-layer \textsf{ReLU}-\textsf{MLP}.

Considering the case where we have an input matrix $X$ with size $(1 + \log (n) + n) \times d$, which can be written as follows, where the superscript denotes the dimension:
\begin{align*}
    X :=
    \begin{bmatrix}
        \begin{array}{cccc}
            r_d^1 & r_d^2 & \cdots & r_d^d\\
            r_a^1 & r_a^2 &\cdots & r_a^d\\
            \hline 
            v_1^1 & v_1^2 &\cdots & v_1^d\\
            v_2^1 & v_2^2 &\cdots & v_2^d\\
            \vdots & \vdots & \ddots & \vdots \\
            v_n^1 & v_n^2 &\cdots & v_n^d\\
        \end{array}
    \end{bmatrix}
\end{align*}

We apply the two layers \ReLUMLP to each column of $X$. Since $X$ has $d$ columns, we loop the two-layer \ReLUMLP for $d$ times. Then, we can perform the ``write'' operation for any $d$-dimension vector from the register to the memory. 
\end{proof}
\section{ADDITION AND SUBTRACTION} \label{sec:app:addt_and_subtract}

In this section, we focus on implementing \ReLUMLP to support basic algorithmic operations, i.e., addition and subtraction.

We begin with introducing the construction fora  one-bit addition. It is worth mentioning that we also consider the carry bit in the addition.
\begin{lemma} [One-Bit Addition, Formal Version of Lemma~\ref{lem:one_bit_addition:informal}] \label{lem:one_bit_addition}
If the following conditions hold:
\begin{itemize}
    \item Let \ReLUMLP be defined as Definition~\ref{def:relu_mlp}.
\end{itemize}

Then, we can show that a $6$-layer \ReLUMLP can emulate the ``addition'' operation for any one-bit data. 
\end{lemma}

\begin{proof}
Consider a simplified case, where we have the following components in the scratchpad: two data register $r_{d_1}, r_{d_2} \in \{\pm 1\}$, and one data register $r_c$ which stores the carry of the addition operation. We want to perform an addition operation and store the result in the first register. Namely, we want the following operation:
\begin{align*}
    x := 
    \begin{bmatrix}
        r_c \\
        \colorbox{lightblue}{$r_{d_1}$} \\
        r_{d_2} \\
    \end{bmatrix}
    \rightarrow
    \begin{bmatrix}
        r_c \\
        \colorbox{lightblue}{$r_{d_1} + r_{d_2}$} \\
        r_{d_2} \\
    \end{bmatrix}
\end{align*}
Here, we want the addition result $r_{d_1} + r_{d_2} \in \{\pm 1\}$ and $r_c = 1$ if and only if $r_{d_1} = r_{d_2} = 1$.

{\bf Step 1: $\{\pm 1\}$ representation to $\{0, 1\}$ representation and reset $r_c$.}

We apply one layer \ReLUMLP with the weight matrix $W_1$ be defined as follows:
\begin{align*}
    W_1 = 
    \begin{bmatrix}
        0 & 0 & 0 \\
        0 & 1 & 0 \\
        0 & 0 & 1 
    \end{bmatrix}
\end{align*}

Our goal is to reset the carry register $r_c$ and change the representation of $r_{d_1}, r_{d_2}$ from $\{\pm 1\}$ to $\{0, 1\}$. Namely, we perform the following operation:
\begin{align*}
    x_1 := \ReLU (W_1 \cdot x) = 
    \begin{bmatrix}
        0 \\
        \ReLU (r_{d_1}) \\
        \ReLU (r_{d_2}) \\
    \end{bmatrix}
\end{align*}
Since $r_{d_1}, r_{d_2} \in \{\pm 1\}$,  $\ReLU (r_{d_1}), \ReLU (r_{d_2}) \in \{0, 1\}$.

{\bf Step 2: Construct two flag vectors for $\ReLU(r_{d_1})$.}

Let $x_{11} := \ReLU(r_{d_1}), x_{12} := \ReLU(r_{d_2})$.

We construct the weight matrix $W_2$ of one \ReLUMLP as follows:
\begin{align*}
    W_2 =
    \begin{bmatrix}
        0 & 0 & 0 \\
        0 & 1 & 0 \\
        0 & 0 & 0 
    \end{bmatrix}, ~~
    b_2 =
    \begin{bmatrix}
        1 \\
        0 \\
        0 
    \end{bmatrix}
\end{align*}

Then, we have
\begin{align*}
    x_2 := W_2 \cdot x_1 + b_2 =
    \begin{bmatrix}
        1 \\
        x_{11} \\
        0 \\
    \end{bmatrix}
\end{align*}

We construct the weight matrix $W_3$ and bias vector $b_3$ of one \ReLUMLP as follows:
\begin{align*}
    W_3 = 
    \begin{bmatrix}
        0 & 0 & 0 \\
        0 & -1 & 0 \\
        0 & 0 & 0 
    \end{bmatrix}, ~~
    b_3 = 
    \begin{bmatrix}
        1 \\
        1 \\
        0
    \end{bmatrix}
\end{align*}

Then, we have
\begin{align*}
    x_3 = W_3 \cdot x_1 + b_3 =
    \begin{bmatrix}
        1 \\
        1 - x_{11} \\
        0
    \end{bmatrix}
\end{align*}

{\bf Step 3: Construct two flag vectors for $\ReLU(r_{d_2})$.}

We construct the weight matrix $W_4$ of one \ReLUMLP as follows:
\begin{align*}
    W_4 =
    \begin{bmatrix}
        0 & 0 & 0 \\
        0 & 0 & 1 \\
        0 & 0 & 0 
    \end{bmatrix}, ~~
    b_4 =
    \begin{bmatrix}
        1 \\
        0 \\
        0
    \end{bmatrix}
\end{align*}

Then, we have
\begin{align*}
    x_4 := W_4 \cdot x_1 + b_4 =
    \begin{bmatrix}
        1 \\
        x_{12} \\
        0 \\
    \end{bmatrix}
\end{align*}

We construct the weight matrix $W_5$ and bias vector $b_5$ of one \ReLUMLP as follows:
\begin{align*}
    W_5 = 
    \begin{bmatrix}
        0 & 0 & 0 \\
        0 & 0 & -1 \\
        0 & 0 & 0 
    \end{bmatrix}, ~~
    b_3 = 
    \begin{bmatrix}
        1 \\
        1 \\
        0
    \end{bmatrix}
\end{align*}

Then, we have
\begin{align*}
    x_5 = W_5 \cdot x_1 + b_3 =
    \begin{bmatrix}
        1 \\
        1 - x_{12} \\
        0
    \end{bmatrix}
\end{align*}

{\bf Step 4: Construct the addition vector.}

Recall in the previous two steps, we defined $x_{11} := \ReLU(r_{d_1}), x_{12} := \ReLU(r_{d_2})$, and we have the following four vectors:
\begin{align*}
    x_2 =
    \begin{bmatrix}
        1 \\
        x_{11} \\
        0
    \end{bmatrix}, ~~
    x_3 =
    \begin{bmatrix}
        1 \\
        1 - x_{11} \\
        0
    \end{bmatrix}, ~~
    x_4 =
    \begin{bmatrix}
        1 \\
        x_{12} \\
        0
    \end{bmatrix}, ~~
    x_5 =
    \begin{bmatrix}
        1 \\
        1 - x_{12} \\
        0
    \end{bmatrix}, ~~
\end{align*}

Then, we construct $x_6$ with the following operation:
\begin{align*}
    x_6 := 
    x_2 \circ x_4 \circ 
    \begin{bmatrix}
        1 \\
        -1 \\
        0
    \end{bmatrix}
    +
    x_2 \circ x_5 \circ 
    \begin{bmatrix}
        -1 \\
        1 \\
        0
    \end{bmatrix}
    +
    x_3 \circ x_4 \circ 
    \begin{bmatrix}
        -1 \\
        1 \\
        0
    \end{bmatrix}
    +
    x_3 \circ x_5 \circ 
    \begin{bmatrix}
        -1 \\
        -1 \\
        0
    \end{bmatrix}
\end{align*}
Our construction is reasonable because only when both $r_{d_1}$ and $r_{d_2}$ are $1$, the carry register $r_c$ will be set to $1$, otherwise it should be set to $-1$.

{\bf Step 5: Erase $r_{d_1}$ and add the addition vector $x_6$.}

We construct the weight matrix $W_6$ as follows:
\begin{align*}
    W_6 =
    \begin{bmatrix}
        1 & 0 & 0 \\
        0 & 0 & 0 \\
        0 & 0 & 1 \\
    \end{bmatrix}
\end{align*}

Then, we perform the following operation:
\begin{align*}
    y := W_6 \cdot x + x_6 =
    \begin{bmatrix}
        r_c \\
        r_{d_1} + r_{d_2} \\
        r_{d_2}
    \end{bmatrix}
\end{align*}

Therefore, we get our addition result.

To sum up, we have
\begin{itemize}
    \item {\bf Step 1} uses one \textsf{ReLU}-\textsf{MLP}.
    \item {\bf Step 2} uses two \textsf{ReLU}-\textsf{MLP}.
    \item {\bf Step 3} uses two \textsf{ReLU}-\textsf{MLP}.
    \item {\bf Step 4} doesn't use \textsf{ReLU}-\textsf{MLP}.
    \item {\bf Step 5} uses one \textsf{ReLU}-\textsf{MLP}.
\end{itemize}

Therefore, we emulate the ``addition'' operation with a six-layer \textsf{ReLU}-\textsf{MLP}.
\end{proof}

Since we already have a \ReLUMLP for one-bit addition, then we extend it to support $d$-bits addition as follows:
\begin{lemma} [$d$-Bits Addition, Formal Version of Lemma~\ref{lem:d_bit_addition:informal}] \label{lem:d_bit_addition}
If the following conditions hold:
\begin{itemize}
    \item Let \ReLUMLP be defined as Definition~\ref{def:relu_mlp}.
\end{itemize}

Then, we can show that a $6$-layer \textsf{ReLU}-\textsf{MLP}, looped for $2d$ times, can emulate the ``addition'' operation for any $d$-dimension vectors. 
\end{lemma}

\begin{proof}
By Lemma~\ref{lem:one_bit_addition}, we have a six-layer \ReLUMLP that can perform a one-bit addition operation and store the carry result in the carry register. 

For each dimension in $d$, we first perform the addition operation between the carry register from the previous dimension with one data register. Then, we perform the addition operation. 

For each dimension, we need $2$ loops. Therefore, we can emulate the $d$-dimension vector addition via $2d$ loops of that six-layer \textsf{ReLU}-\textsf{MLP}. 
\end{proof}

In this work, we use $2$'s-complement as the representation of data. Therefore, the subtraction can be viewed as first negating the subtrahend, then adding $1$, and adding to the minuend. We follow the aforementioned high-level idea to implement the subtraction as follows:
\begin{lemma}[$d$-Bits Subtraction, Formal Version of Lemma~\ref{lem:subtraction:informal}] \label{lem:subtraction}
If the following conditions hold:
\begin{itemize}
    \item Let \ReLUMLP be defined as Definition~\ref{def:relu_mlp}.
\end{itemize}

Then, we can show that a $7$-layer \ReLUMLP can emulate the ``subtraction'' operation for any $d$-dimension vectors. 
\end{lemma}

\begin{proof}

By Lemma~\ref{lem:d_bit_addition}, we have we can perform addition operation for $d$-dimension vectors via a six-layer \ReLUMLP with $2d$ loops.

In our setting, we use $2$'s complement (Definition~\ref{def:complement}) to represent our data. Therefore, to perform subtraction $x_1 - x_2$, we can first calculate $-x_2$, then add $x_1$ and $-x_2$ to get the final result.

{\bf Step 1: Calculate $-x_2$.}

According to $2$'s complement, to calculate $-x_2$, we first need to invert all entries of $x_2$ ($1 \rightarrow -1; -1 \rightarrow 1$). This step requires a one-layer \ReLUMLP with weight matrix $W = - I_{d \times d}$. 

The second step is to add $1$ to the inverted $x_2$, which can be achieved by the six-layer \ReLUMLP used for the addition operation.

{\bf Step 2: Add $x_1$ and $-x_2$.}

This step keeps uses the same six-layer \ReLUMLP with {\bf Step 1}. Since that six-layer \ReLUMLP can perform addition operation, we can use that six-layer \ReLUMLP to perform $x_1 + (-x_2)$, which is our desired result.  

To sum up, we need an additional layer to invert $x_2$ as discussed in {\bf Step 1}, and the six-layer \ReLUMLP is used for the addition operation. Therefore, the subtraction operation requires a seven-layer \textsf{ReLU}-\textsf{MLP}. 
\end{proof}
\section{CONDITIONAL BRANCHING} \label{sec:app:conditional_branching}

In this section, we implement the ``conditional branching'' instruction through \textsf{ReLU-MLP}. 
Conditional branching is a critical instruction in computer programs since it enables the computer programs to achieve controllability.

\begin{lemma}[Conditional Branching, Formal Version of Lemma~\ref{lem:conditional_branching:informal}] \label{lem:conditional_branching}
If the following conditions hold:
\begin{itemize}
    \item Let \ReLUMLP be defined as Definition~\ref{def:relu_mlp}.
    \item Let $n$ denote the number of data in the memory. For $i \in [n]$, each data $v_i \in \{ \pm 1 \}$. 
    \item Let the address vector $a_i \in \{\pm 1\}^{\log(n)}$.
\end{itemize}

Then, we can show that a
$4$-layer \ReLUMLP can emulate the ``conditional branching'' operation.
\end{lemma}

\begin{proof}
Consider a simplified case where we have the following components in the scratchpad: a flag register stores $\flag \{\pm 1\}$. If we jump ($\flag = 1$), the program counter register $r_{pc} \in \{\pm 1\}^{\log (n)}$, points to a target address register $r_{a_b} \in \{\pm 1\}^{\log (n)}$. If we stay ($\flag = -1$), the program counter register $r_{pc}$ will be incremented by one, which means it will point to an address register $r_{pc+1} \in \{\pm 1\}^{\log (n)}$, where the $r_{pc+1}$ can be calculated by Lemma~\ref{lem:d_bit_addition} with adding $r_{pc}$ with $1$. 

Namely, if $\flag = 1$, we want the following operation:
\begin{align*}
    x = 
    \begin{bmatrix}
        \flag \\
        \colorbox{lightblue}{$r_{pc}$} \\
        r_{pc+1} \\
        r_{a_b}
    \end{bmatrix}
    \rightarrow
    \begin{bmatrix}
        \flag \\
        \colorbox{lightblue}{$r_{a_b}$} \\
        r_{pc+1} \\
        r_{a_b}
    \end{bmatrix}
\end{align*}

If $\flag = -1$, we want the following operation:
\begin{align*}
    x = 
    \begin{bmatrix}
        \flag \\
        \colorbox{lightblue}{$r_{pc}$} \\
        r_{pc+1} \\
        r_{a_b}
    \end{bmatrix}
    \rightarrow
    \begin{bmatrix}
        \flag \\
        \colorbox{lightblue}{$r_{pc+1}$} \\
        r_{pc+1} \\
        r_{a_b}
    \end{bmatrix}
\end{align*}

{\bf Step 1: Extract $r_{a_b}$.}

We construct the following weight matrix $W_1 \in \R^{(3 \log(n)+1) \times (3 \log(n)+1)}$:
\begin{align*}
    W_1 = 
    \begin{bmatrix}
        0 & \0_{\log(n)}^\top & \0_{\log(n)}^\top & \0_{\log(n)}^\top \\
        \0_{\log(n)} & \0_{\log(n) \times \log(n)} & I_{\log(n) \times \log(n)} & \0_{\log(n) \times \log(n)} \\
        \0_{\log(n)} & \0_{\log(n) \times \log(n)} & \0_{\log(n) \times \log(n)} & \0_{\log(n) \times \log(n)} \\
        \0_{\log(n)} & \0_{\log(n) \times \log(n)} & \0_{\log(n) \times \log(n)} & \0_{\log(n) \times \log(n)} \\
    \end{bmatrix}
\end{align*}

Then we have
\begin{align*}
    x_1 := W_1 \cdot x = 
     \begin{bmatrix}
        0 \\
        r_{a_b} \\
        \0_{\log(n)} \\
        \0_{\log(n)}
    \end{bmatrix}
\end{align*}

{\bf Step 2: Extract $r_{pc+1}$.}

We construct the following weight matrix $W_2 \in \R^{(3 \log(n)+1) \times (3 \log(n)+1)}$:
\begin{align*}
    W_2 = 
    \begin{bmatrix}
        0 & \0_{\log(n)}^\top & \0_{\log(n)}^\top & \0_{\log(n)}^\top \\
        \0_{\log(n)} & \0_{\log(n) \times \log(n)} & \0_{\log(n) \times \log(n)} & I_{\log(n) \times \log(n)} \\
        \0_{\log(n)} & \0_{\log(n) \times \log(n)} & \0_{\log(n) \times \log(n)} & \0_{\log(n) \times \log(n)} \\
        \0_{\log(n)} & \0_{\log(n) \times \log(n)} & \0_{\log(n) \times \log(n)} & \0_{\log(n) \times \log(n)} \\
    \end{bmatrix}
\end{align*}

Then we have
\begin{align*}
    x_2 := W_2 \cdot x = 
     \begin{bmatrix}
        0 \\
        r_{pc+1} \\
        \0_{\log(n)} \\
        \0_{\log(n)}
    \end{bmatrix}
\end{align*}

{\bf Step 3: Simulate the flag.}

We use \ReLUMLP to extract $\flag \in \{\pm 1\}$ and perform the following operation:
We construct the weight matrix $W_3 \in \R^{(3 \log(n)+1) \times (3 \log(n)+1)}$:
\begin{align*}
    W_3 = 
    \begin{bmatrix}
        0 & \0_{\log(n)}^\top & \0_{\log(n)}^\top & \0_{\log(n)}^\top \\
        \1_{\log(n)} & \0_{\log(n) \times \log(n)} & \0_{\log(n) \times \log(n)} & \0_{\log(n) \times \log(n)} \\
        \0_{\log(n)} & \0_{\log(n) \times \log(n)} & \0_{\log(n) \times \log(n)} & \0_{\log(n) \times \log(n)} \\
        \0_{\log(n)} & \0_{\log(n) \times \log(n)} & \0_{\log(n) \times \log(n)} & \0_{\log(n) \times \log(n)} \\
    \end{bmatrix}
\end{align*}

We apply $W_3$ to $x$, we have
\begin{align*}
    W_3 \cdot x = 
    \begin{bmatrix}
        0 \\
        \flag \cdot \1_{\log (n)} \\
        \0_{\log(n)} \\
        \0_{\log(n)} \\
    \end{bmatrix}
\end{align*}

Then, we have
\begin{align*}
     x_3 := \ReLU (W_3 x) \circ x_1 + \ReLU (-1 \cdot (W_3 x)) \circ x_2
\end{align*}

{\bf Step 4: Erase $r_{pc}$ and repoint.}
We construct the weight matrix $W_4 \in \R^{(3 \log(n)+1) \times (3 \log(n)+1)}$:
\begin{align*}
    W_4 = 
    \begin{bmatrix}
        1 & \0_{\log(n)}^\top & \0_{\log(n)}^\top & \0_{\log(n)}^\top \\
        \0_{\log(n)} & \0_{\log(n) \times \log(n)} & \0_{\log(n) \times \log(n)} & \0_{\log(n) \times \log(n)} \\
        \0_{\log(n)} & \0_{\log(n) \times \log(n)} & I_{\log(n) \times \log(n)} & \0_{\log(n) \times \log(n)} \\
        \0_{\log(n)} & \0_{\log(n) \times \log(n)} & \0_{\log(n) \times \log(n)} & I_{\log(n) \times \log(n)} \\
    \end{bmatrix}
\end{align*}

Then we have
\begin{align*}
    y := W_4 x + x_3
\end{align*}

Then, $y$ is the final output we want.

To sum up, we have
\begin{itemize}
    \item {\bf Step 1} uses one \ReLUMLP with width $3 \log (n) + 1$.
    \item {\bf Step 2} uses one \ReLUMLP with width $3 \log (n) + 1$.
    \item {\bf Step 3} uses one \ReLUMLP with width $3 \log (n) + 1$.
    \item {\bf Step 4} uses one \ReLUMLP with width $3 \log (n) + 1$.
\end{itemize}

Therefore, we use \ReLUMLP with four layers and width $O(\log n)$ to emulate the ``conditional branching'' operation.
\end{proof}
\section{SUBLEQ} \label{sec:app:subleq}

Based on the fundamental operations introduced in the previous sections, we are ready to construct the ``\SUBLEQ'' instruction. 

\begin{lemma}[\ReLUMLP Emulate \SUBLEQ, Formal Version of Lemma~\ref{lem:subleq:informal}]
\label{lem:subleq}
If the following conditions hold:
\begin{itemize}
    \item Let \ReLUMLP be defined as Definition~\ref{def:relu_mlp}.
    \item Let $n$ denote the size of the state vector. 
    \item Let $m$ denote the number of instructions. 
    \item Let $k$ denote the number of one-bit data stored in the memory. For $i \in [k]$, each data is $v_i \in \{ \pm 1 \}$ and the memory size $k$ satisfies $k = n - 2 - 4 \log(n) - 3 m \log(n)$.
    \item Let the address vector $a_i \in \{\pm 1\}^{\log(n)}$.
    \item Let the instruction $c_i \in \{ \pm 1 \}^{3 \log (n)}$ be defined as Definition~\ref{def:instruction}.
    \item Suppose we have three data registers $r_c, r_{d_1}, r_{d_2} \in \{\pm 1 \}$, one carry bit $r_c \in \{\pm 1 \}$, three address registers $r_{a_1}, r_{a_2}, r_{a_3} \in \{ \pm 1 \}^{\log (n)}$, and one program counter $r_{pc} \in \{ \pm 1 \}^{\log (n)}$ in the scratchpad. 
\end{itemize}

Then, we can show that, a $23$-layer \ReLUMLP with width $n$ can emulate the ``\SUBLEQ'' operation (Algorithm~\ref{alg:subleq}).
\end{lemma}

\begin{proof}
Consider we organize our state vector as follows:
\begin{align*}
    x =
    \begin{bmatrix}
        \begin{array}{cc}
            r_c \\
            r_{d_1} \\
            r_{d_2} \\
            r_{a_1} \\
            r_{a_2} \\
            r_{a_3} \\
            r_{pc} \\
            \hline
            v_1 \\
            v_2 \\
            \vdots \\
            v_k \\
            \hline 
            c_1 \\
            c_2 \\
            \vdots \\
            c_{m-1} \\
            c_{\EOF}
        \end{array}
    \end{bmatrix}
\end{align*}
Here, $r_c \in \{\pm 1\}$ denote the carry bit, $r_{d_1}, r_{d_1} \in \{\pm 1\}$ denote two data registers; $r_{a_1}, r_{a_2}, r_{a_3} \in \{\pm 1\}^{\log (n)}$ denote three address registers; $r_{pc} \in \{\pm 1\}^{\log (n)}$ denotes the program counter; $v_1, v_2, \cdots, v_k \in \{\pm 1\}$ denote $k$ one-bit data stored in the memory; $c_1, c_2, \cdots, c_m \in \{\pm 1\}^{3 \log (n)}$ denote $m$ instructions, where $c_m = c_{\EOF}$, denoting the End Of File (EOF) instruction, which means the program should terminate here.
Since the memory size $k$ satisfies $k = n - 2 - 4 \log(n) - 3 m \log(n)$, the total length of the state vector is $n$. 

{\bf Step 1: Read the three addresses.}

In this step, we read the instruction from the memory according to the address provided in $r_{pc}$. As shown in Algorithm~\ref{alg:subleq}, the instruction contains three addresses. Here we denote them as $a, b, c \in \{\pm 1\}^{\log(n)}$. 

By Lemma~\ref{lem:read_d_bits}, we read the three address vectors from the memory to the three registers in the scratchpad, using two-layer \textsf{ReLU}-\textsf{MLP}. Then, the state $x$ transforms as follows:
\begin{align*}
    x =
    \begin{bmatrix}
        \begin{array}{cc}
            r_c \\
            r_{d_1} \\
            r_{d_2} \\
            \colorbox{lightblue}{$r_{a_1}$} \\
            \colorbox{lightgreen}{$r_{a_2}$} \\
            \colorbox{lightred}{$r_{a_3}$} \\
            r_{pc} \\
            \hline
            v_1 \\
            v_2 \\
            \vdots \\
            v_k \\
            \hline 
            c_1 \\
            c_2 \\
            \vdots \\
            c_{m-1} \\
            c_{\EOF}
        \end{array}
    \end{bmatrix}
    \rightarrow
    \begin{bmatrix}
        \begin{array}{cc}
            r_c \\
            r_{d_1} \\
            r_{d_2} \\
            \colorbox{lightblue}{$a$} \\
            \colorbox{lightgreen}{$b$} \\
            \colorbox{lightred}{$c$} \\
            r_{pc} \\
            \hline
            v_1 \\
            v_2 \\
            \vdots \\
            v_k \\
            \hline 
            c_1 \\
            c_2 \\
            \vdots \\
            c_{m-1} \\
            c_{\EOF}
        \end{array}
    \end{bmatrix}
\end{align*}

Overall, this step requires a two-layer \textsf{ReLU}-\textsf{MLP}.

{\bf Step 2: Read the data required by the instruction.}

In this step, we read the two data required by the instruction. Namely $\mem[a]$ and $\mem[b]$. By Lemma~\ref{lem:read_d_bits}, we achieve this operation by a two-layer \textsf{ReLU}-\textsf{MLP}. Then, the state $x$ transforms as follows:
\begin{align*}
    x = 
    \begin{bmatrix}
        \begin{array}{cc}
            r_c \\
            \colorbox{lightblue}{$r_{d_1}$} \\
            \colorbox{lightgreen}{$r_{d_2}$} \\
            a \\
            b \\
            c \\
            r_{pc} \\
            \hline
            v_1 \\
            v_2 \\
            \vdots \\
            v_k \\
            \hline 
            c_1 \\
            c_2 \\
            \vdots \\
            c_{m-1} \\
            c_{\EOF}
        \end{array}
    \end{bmatrix}
    \rightarrow
    \begin{bmatrix}
        \begin{array}{cc}
            r_c \\
            \colorbox{lightblue}{$\mem[a]$} \\
            \colorbox{lightgreen}{$\mem[b]$} \\
            a \\
            b \\
            c \\
            r_{pc} \\
            \hline
            v_1 \\
            v_2 \\
            \vdots \\
            v_k \\
            \hline 
            c_1 \\
            c_2 \\
            \vdots \\
            c_{m-1} \\
            c_{\EOF}
        \end{array}
    \end{bmatrix}
\end{align*}

Overall, this step requires a two-layer \textsf{ReLU}-\textsf{MLP}.

{\bf Step 3: Perform subtraction.}

In this step, we perform the subtraction operation, $\mem[b] - \mem[a]$. By Lemma~\ref{lem:subtraction}, we achieve this operation by a seven-layer \textsf{ReLU}-\textsf{MLP}. Then, the state $x$ transforms as follows:
\begin{align*}
    x = 
    \begin{bmatrix}
        \begin{array}{cc}
            r_c \\
            \colorbox{lightblue}{$\mem[a]$} \\
            \mem[b] \\
            a \\
            b \\
            c \\
            r_{pc} \\
            \hline
            v_1 \\
            v_2 \\
            \vdots \\
            v_k \\
            \hline 
            c_1 \\
            c_2 \\
            \vdots \\
            c_{m-1} \\
            c_{\EOF}
        \end{array}
    \end{bmatrix}
    \rightarrow
    \begin{bmatrix}
        \begin{array}{cc}
            r_c \\
            \colorbox{lightblue}{$\mem[b] - \mem[a]$} \\
            \mem[b] \\
            a \\
            b \\
            c \\
            r_{pc} \\
            \hline
            v_1 \\
            v_2 \\
            \vdots \\
            v_k \\
            \hline 
            c_1 \\
            c_2 \\
            \vdots \\
            c_{m-1} \\
            c_{\EOF}
        \end{array}
    \end{bmatrix}
\end{align*}

This step requires a seven-layer \textsf{ReLU}-\textsf{MLP}.

{\bf Step 4: Write back $\mem[b] - \mem[a]$.}

In this step, we write back the $\mem[b] - \mem[a]$ to the memory according to the address vector $b$. By Lemma~\ref{lem:write_d_bits}, we achieve this operation via a two-layer \textsf{ReLU}-\textsf{MLP}.

\begin{align*}
    x = 
    \begin{bmatrix}
        \begin{array}{cc}
            r_c \\
            \mem[b] - \mem[a] \\
            \mem[b] \\
            a \\
            b \\
            c \\
            r_{pc} \\
            \hline
            v_1 \\
            v_2 \\
            \vdots \\
            v_{b-1} \\
            \colorbox{lightblue}{$v_b$} \\
            v_{b+1} \\
            \vdots \\
            v_k \\
            \hline 
            c_1 \\
            c_2 \\
            \vdots \\
            c_{m-1} \\
            c_{\EOF}
        \end{array}
    \end{bmatrix}
    \rightarrow
    \begin{bmatrix}
        \begin{array}{cc}
            r_c \\
            \mem[b] - \mem[a] \\
            \mem[b] \\
            a \\
            b \\
            c \\
            r_{pc} \\
            \hline
            v_1 \\
            v_2 \\
            \vdots \\
            v_{b-1} \\
            \colorbox{lightblue}{$\mem[b] - \mem[a]$} \\
            v_{b+1} \\
            \vdots \\
            v_k \\
            \hline 
            c_1 \\
            c_2 \\
            \vdots \\
            c_{m-1} \\
            c_{\EOF}
        \end{array}
    \end{bmatrix}
\end{align*}

This step requires a two-layer \textsf{ReLU}-\textsf{MLP}.

{\bf Step 5: Calculate $r_{pc+1}$.}

In this step, we calculate $r_{pc+1}$ and store it at the place which stores $b$ previously. (Since the address vectors $a$ and $b$ will not be used in the following steps, and we can overwrite them with any other address vectors.) By Lemma~\ref{lem:d_bit_addition}, we achieve this operation via a six-layer \textsf{ReLU}-\textsf{MLP}. Then, the state $x$ transforms as follows:
\begin{align*}
    x = 
    \begin{bmatrix}
        \begin{array}{cc}
            r_c \\
            \mem[b] - \mem[a] \\
            \mem[b] \\
            a \\
            \colorbox{lightblue}{$b$} \\
            c \\
            r_{pc} \\
            \hline
            v_1 \\
            v_2 \\
            \vdots \\
            v_k \\
            \hline 
            c_1 \\
            c_2 \\
            \vdots \\
            c_{m-1} \\
            c_{\EOF}
        \end{array}
    \end{bmatrix}
    \rightarrow
    \begin{bmatrix}
        \begin{array}{cc}
            r_c \\
            \mem[b] - \mem[a] \\
            \mem[b] \\
            a \\
            \colorbox{lightblue}{$r_{pc+1}$} \\
            c \\
            r_{pc} \\
            \hline
            v_1 \\
            v_2 \\
            \vdots \\
            v_k \\
            \hline 
            c_1 \\
            c_2 \\
            \vdots \\
            c_{m-1} \\
            c_{\EOF}
        \end{array}
    \end{bmatrix}
\end{align*}

This step requires a six-layer \textsf{ReLU}-\textsf{MLP}.

{\bf Step 6: Conditional branching.}

In this step, we perform the conditional branching operation according to $\mem[b] - \mem[a]$. For simplicity, let $\flag := \mem[b] - \mem[a]$. We use $r_{\target} \in \{ r_{pc+1}, c \}$ to denote the final address stored in the program counter $r_{pc}$. Namely, if $\flag = 1$, $r_{\target} = c$; if $\flag = -1$, $r_{\target} = r_{pc+1}$. 
By Lemma~\ref{lem:conditional_branching}, we achieve this operation via a four-layer \textsf{ReLU}-\textsf{MLP}. Then, the state $x$ transforms as follows:
\begin{align*}
    x = 
    \begin{bmatrix}
        \begin{array}{cc}
            r_c \\
            \flag \\
            \mem[b] \\
            a \\
            r_{pc+1} \\
            c \\
            \colorbox{lightblue}{$r_{pc}$} \\
            \hline
            v_1 \\
            v_2 \\
            \vdots \\
            v_k \\
            \hline 
            c_1 \\
            c_2 \\
            \vdots \\
            c_{m-1} \\
            c_{\EOF}
        \end{array}
    \end{bmatrix}
    \rightarrow
    \begin{bmatrix}
        \begin{array}{cc}
            r_c \\
            \flag \\
            \mem[b] \\
            a \\
            r_{pc+1} \\
            c \\
            \colorbox{lightblue}{$r_{\target}$} \\
            \hline
            v_1 \\
            v_2 \\
            \vdots \\
            v_k \\
            \hline 
            c_1 \\
            c_2 \\
            \vdots \\
            c_{m-1} \\
            c_{\EOF}
        \end{array}
    \end{bmatrix}
\end{align*}

This operation requires a four-layer \textsf{ReLU}-\textsf{MLP}. 

{\bf Step 7: Program termination.}

The special instruction $c_{\EOF}$ is used to signal the end of the program. Similar to other instructions, $c_{\EOF}$ also consists of three address vectors. 

To terminate the program, our idea is to make $r_{pc}$ always point to $c_{\EOF}$ after executing $c_{\EOF}$. For the three address vectors in $c_{\EOF}$, we set the $a = b$, and $c = \&(c_{\EOF})$, where $\&(c_{\EOF})$ denotes the address of the $c_{\EOF}$ instruction. This design is reasonable because setting $a = b$ brings $\mem[a] = \mem[b]$. Therefore, the condition always holds in the ``\SUBLEQ'' instruction. Hence, the program counter $r_{pc}$ will always jump to $c$. Since we have set $c$ to the address of $c_{\EOF}$, the $r_{pc}$ still points to $c_{\EOF}$ after execute $c_{\EOF}$.

To sum up, we have
\begin{itemize}
    \item {\bf Step 1: Read the three addresses} requires a two-layer \textsf{ReLU}-\textsf{MLP}.
    \item {\bf Step 2: Read the data required by the instruction} requires a two-layer \textsf{ReLU}-\textsf{MLP}.
    \item {\bf Step 3: Perform subtraction} requires a seven-layer \textsf{ReLU}-\textsf{MLP}.
    \item {\bf Step 4: Write back $\mem[b] - \mem[a]$} requires a two-layer \textsf{ReLU}-\textsf{MLP}.
    \item {\bf Step 5: Calculate $r_{pc+1}$} requires an six-layer \textsf{ReLU}-\textsf{MLP}.
    \item {\bf Step 6: Conditional branching} requires a four-layer \textsf{ReLU}-\textsf{MLP}.
\end{itemize}

Since we always operate the state vector $x \in \R^n$, then the width of the above-mentioned \ReLUMLP is $n$.

Therefore, to emulate the ``\SUBLEQ'' instruction, we need a twenty-three-layer \ReLUMLP with size $n \times n$.
\end{proof}

\section{LOOPED \textsc{ReLU-MLP} AS PROGRAMMABLE COMPUTER} \label{sec:app:looped_relu_mlp_as_programmable_computer}

Finally, in this section, we demonstrate the One Instruction Set Computer (OISC) constructed by the ``\SUBLEQ'' instruction is equivalent to the programmable computer in terms of computational ability, which further indicates that the $23$-layer \ReLUMLP discussed in the previous section is capable of functioning as a programmable computer.

\begin{theorem} [Looped \ReLUMLP as Programmable Computer, Formal Version of Theorem~\ref{thm:looped_relu_mlp_as_programmable_computer:informal}] \label{thm:looped_relu_mlp_as_programmable_computer}
If the following conditions hold:
\begin{itemize}
    \item Let \ReLUMLP be defined as Definition~\ref{def:relu_mlp}.
    \item Let $n$ denote the size of the state vector. 
    \item Let $m$ denote the number of instructions. 
    \item Let $k$ denote the number of one-bit data stored in the memory. For $i \in [k]$, each data is $v_i \in \{ \pm 1 \}$ and the memory size $k$ satisfies $k = n - 2 - 4 \log(n) - 3 m \log(n)$.
    \item Let the address vector $a_i \in \{\pm 1\}^{\log(n)}$.
    \item Let the instruction $c_i \in \{ \pm 1 \}^{3 \log (n)}$ be defined as Definition~\ref{def:instruction}.  
    
    \item Suppose we have two data registers $r_{d_1},  r_{d_2} \in \{\pm 1 \}$, one carry bit $r_c \in \{\pm 1 \}$, three address registers $r_{a_1}, r_{a_2}, r_{a_3} \in \{ \pm 1 \}^{\log (n)}$, and one program counter $r_{pc} \in \{ \pm 1 \}^{\log (n)}$ in the scratchpad. 
\end{itemize}

Then, we can show that a $23$-layer \ReLUMLP with width $n$ can emulate a programmable computer, where $d$ is the number of bits we use to store each integer. Namely, this ``computer'' supports integers within the range $[-2^{d-1}, 2^{d-1} - 1]$. 
\end{theorem}

\begin{proof}
In Lemma~\ref{lem:subleq}, we have proved that a $23$-layer \ReLUMLP is capable of emulating the one-bit version of ``\SUBLEQ'' instruction. 

{\bf Step 1: Extend to $d$-bits ``\SUBLEQ''.}

We demonstrate how the $1$-bit SUBLEQ instruction can be extended to a $d$-bit SUBLEQ instruction. 
Similar to the proof of Lemma~\ref{lem:read_d_bits} and \ref{lem:write_d_bits}, we can extend horizontally from the one-bit version to $d$-bits version, where we apply the $23$-layer \ReLUMLP row by row, with totally $d$ loops. 

To be more specific, each operation on the $d$-bit data can be decomposed into $d$ individual operations on $1$-bit data. Our design for the $1$-bit operations takes this into account. Specifically, the $1$-bit addition operation (Lemma~\ref{lem:one_bit_addition:informal}) also outputs a carry bit, indicating whether the operation results in a carry. By simply stacking the $1$-bit operations $d$ times, we can effectively support the entire $d$-bit operation.

{\bf Step 2: Extend to OISC.}

By \citet{mp88}, the One Instruction Set Computer (OISC) with the instruction ``\SUBLEQ'' is Turing complete, which means it can compute arbitrary programs. Therefore, we can conclude that our looped $23$-layer \ReLUMLP can actually function as a programmable computer. 
 
\end{proof}

Note that \citet{mp88} shows that anything a programmable computer can do can also be accomplished by stacking ``\SUBLEQ'' instructions. This means that, within our framework, any computation performed by a programmable computer can also be executed by looping our 23-layer ReLU-MLP with enough iterations.




\end{document}